\pdfoutput=1

\documentclass[11pt]{article}

\usepackage[final]{acl}

\usepackage{times}
\usepackage{latexsym}

\usepackage[T1]{fontenc}

\usepackage[utf8]{inputenc}

\usepackage{microtype}

\usepackage{inconsolata}

\usepackage{amsthm}
\newtheorem{proposition}{Proposition}
\usepackage{fontawesome5}
\usepackage{graphicx}
\usepackage{xcolor}
\usepackage{mdframed}
\usepackage[most]{tcolorbox}
\usepackage{subcaption} 
\usepackage{hyperref}
\usepackage{booktabs}
\usepackage{algorithm}
\usepackage[noend]{algorithmic}
\usepackage{amssymb}
\usepackage{CJKutf8}
\usepackage{colortbl}
\usepackage{multirow} 
\usepackage{amsmath}

\title{Anchored Sliding Window: Toward Robust and Imperceptible\\ Linguistic Steganography}

\author{
  Ruiyi Yan\textsuperscript{1} \quad
  Shiao Meng\textsuperscript{2} \quad
  Yugo Murawaki\textsuperscript{1} \\[6pt]
  \textsuperscript{1}Graduate School of Informatics, Kyoto University \\
  \textsuperscript{2}School of Software, Tsinghua University \\[4pt]
  \texttt{ruiyi@nlp.ist.i.kyoto-u.ac.jp}, \texttt{msa21@mails.tsinghua.edu.cn} \\
  \texttt{murawaki@i.kyoto-u.ac.jp}
}

\begin{document}
\maketitle
\begin{abstract}
Linguistic steganography based on language models typically assumes that steganographic texts are transmitted without alteration, making them fragile to even minor modifications.
While previous work mitigates this fragility by limiting the context window, it significantly compromises text quality.
In this paper, we propose the \textbf{anchored sliding window (ASW)} framework to improve imperceptibility and robustness. 
In addition to the latest tokens, the prompt and a \textbf{bridge context} are anchored within the context window, encouraging the model to \textit{compensate for the excluded tokens}.
We formulate the optimization of the bridge context as a variant of \textbf{prompt distillation}, which we further extend using self-distillation strategies.
Experiments show that our ASW significantly and consistently outperforms the baseline method in text quality, imperceptibility, and robustness across diverse settings.
The code is available at \faGithub~\href{https://github.com/ryehr/ASW\_steganography}{github.com/ryehr/ASW\_steganography}.
\end{abstract}

\section{Introduction}
As the demand for private communication to evade surveillance grows, users often turn to encryption. However, encrypted messages with unreadable content can be easily detected and are likely to be blocked in repressive environments~\cite{raman2020measuring, 10.1145/3460120.3484550}, as such communication is deemed suspicious. In contrast, steganography, which embeds messages into seemingly innocuous carriers such as images~\cite{9335027} or text~\cite{math9212829}, offers a more suitable approach against undue surveillance.

Among steganographic techniques, linguistic steganography, particularly methods based on large language models (LLMs), has attracted increasing attention in recent years~\cite{8470163, 9193914, 10.1145/3664647.3680562,yan2025comprehensive},  driven by the ubiquity of text on social media and the high-quality text generation enabled by rapidly advancing LLMs~\cite{NEURIPS2020_1457c0d6,achiam2023gpt}. 
In the classic steganographic scenario known as the ``Prisoners’ Problem''~\cite{Simmons1984}, Alice and Bob (the steganographers) attempt to exchange an escape plan while their communication is monitored by a warden, Eve (the steganalyzer), who will block communication if any message is deemed suspicious. Thus, in linguistic steganography, desirable steganographic texts (\textit{stegotexts}) are \textit{imperceptible} from normal texts~\cite{9193914, 9353234}.

\begin{figure}[!t]
 \centering
 \includegraphics[width=\columnwidth]{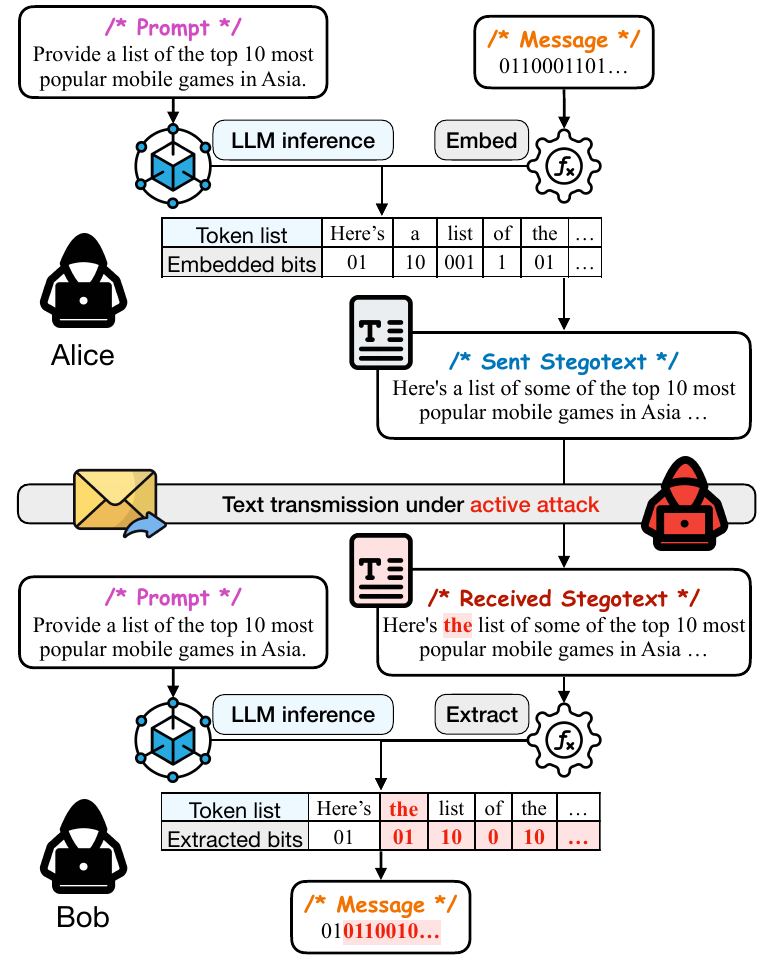} 
 \caption{Steganographic extraction is fragile to active attack, even when only a single token of the stegotext is modified. In this case, modifying the second token disrupts all subsequent autoregressive inferences, thus affecting their extracted results.}
 \label{fig: Stego_under_attack}
\end{figure}

\begin{figure*}[!t]
    \centering
    \begin{subfigure}[b]{0.32\textwidth}
        \centering
        \includegraphics[width=\linewidth]{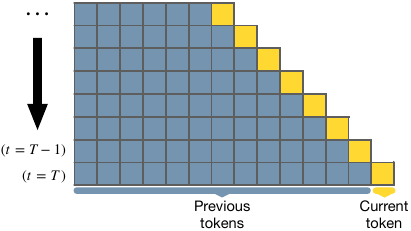} 
        \caption{Conventional robustness-\textbf{unaware} full context used in most methods.}
        \label{fig: full_context} 
    \end{subfigure}
    \hfill 
    \begin{subfigure}[b]{0.32\textwidth}
        \centering
        \includegraphics[width=\linewidth]{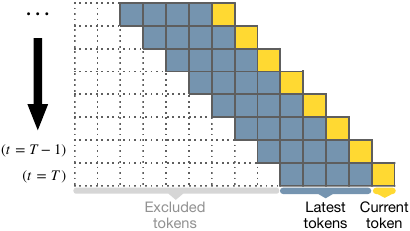} 
        \caption{Basic robustness-\textbf{aware} context window (e.g., used in~\citet{10888944}).}
        \label{fig: WinStega_context} 
    \end{subfigure}
    \hfill 
    \begin{subfigure}[b]{0.32\textwidth}
        \centering
        \includegraphics[width=\linewidth]{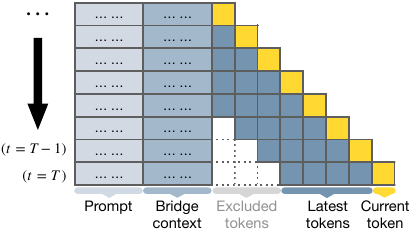}
        \caption{Refined robustness-\textbf{aware} context window used in our ASW framework.}
        \label{fig: ASW_context} 
    \end{subfigure}
    \caption{Comparison between the context window of our ASW and that of other methods. Each row corresponds to the sequence of generated tokens. The bottom row represents the current generation step ($t = T$), the row above corresponds to the previous step ($t = T -1$), an so forth.}
    \label{fig: Context_window} 
\end{figure*}

In addition to imperceptibility, robustness is also a crucial aspect of steganography~\cite{4663056, 5329453}. In the scope of \textit{symmetric} linguistic steganography (both Alice and Bob access the same language model and the same prompt for embedding and extraction), even a minor modification to a stegotext can distort the subsequent extraction relying on the subsequent inference. Figure~\ref{fig: Stego_under_attack} illustrates the fragility of steganographic extraction in robustness-unaware methods when stegotexts are subjected to \textit{active attacks}~\cite{backes2005public}, a vulnerability rooted in the \textit{autoregressive} inferences of LLMs.

Previous work~\cite{10888944} has been designed to improve robustness against external alterations (a form of active attack), such as substitution, deletion, and insertion. Its approach is to truncate the original context window and rely solely on the latest tokens for inference (as illustrated in Figure~\ref{fig: WinStega_context}). This design ensures that any alteration outside the latest tokens does not affect inference or extraction.
However, such a restricted context window causes text quality to degrade sharply because initial tokens are excluded~\cite{xiao2024efficient}.

Motivated to ensure robustness as well as text quality and imperceptibility of stegotexts, we propose the \textbf{ASW} (\textbf{A}nchored \textbf{S}liding \textbf{W}indow) framework, 
in which the context window is composed of the \textbf{prompt}, the \textbf{bridge context}, and the \textbf{latest tokens} (see Section~\ref{sec: Anchored Sliding Window (ASW)}). Figure~\ref{fig: Context_window} compares the context windows in our ASW and other methods.

For the bridge context, we start with an intuition: the \textbf{hard} bridge context (referring to discrete tokens, with an example in Figure~\ref{fig: Hard_Bridge_Context}). It is designed to help the model \textit{imagine} and compensate for the excluded part, to mitigate the divergence between inference based on the full context and ASW-based inference, thereby enhancing imperceptibility (see Section~\ref{sec: Intuition: Hard Bridge Context}).

Further, we propose a tunable \textbf{soft} bridge context (referring to continuous embedding vectors), optimized by a self-distillation framework, whose goal is to minimize the divergence between full-context and ASW-based inference (see Section~\ref{sec: Tunable Soft Bridge Context via Self-Distillation}).

Experiments (see Section~\ref{sec: Experiments}) show that ASW substantially outperforms the baseline method, in \textbf{imperceptibility}, \textbf{robustness}, and \textbf{text quality} (e.g., improving ROUGE-L by 104.08\% under Qwen2.5-7B-Instruct~\cite{qwen2.5}). We also evaluate ASW across hyperparameter, model, and dataset variants to demonstrate its \textbf{generalizability}.


\section{Background and Related Work}
\subsection{Language Model Basics}
A language model (LM) has a vocabulary $\mathcal{V}$ consisting of \textit{tokens}.
Tokens with negative indices, $[s^{(-N_p)},\dots,s^{(-1)}]$, represent a \textit{prompt} of length $N_p$ and $[s^{(0)},\dots,s^{(T-1)}]$ are tokens generated by an LM in response to the prompt.

LM next-token prediction at position $t$ is a function defined as:
\begin{equation}
    f_{\mathrm{LM}}(\boldsymbol{s}^{(t-1)};\theta) = \boldsymbol{l}^{(t)}
\end{equation}
in which $\boldsymbol{s}^{(t-1)}$ can be the full context $[s^{(-N_p)},\dots,s^{(t-1)}]$ at $t-1$, $\theta$ is the model parameters, and $\boldsymbol{l}^{(t)} = (l^{(t)}_1,\dots,l^{(t)}_{|\mathcal{V}|}) \in \mathbb{R}^{|\mathcal{V}|}$ is the logit vector, corresponding to each token 
in $\mathcal{V}$.
This discrete probability distribution can be calculated through the logit vector and the softmax function, i.e.,  $\boldsymbol{p}^{(t)} =(p^{(t)}_1,\dots,p^{(t)}_{|\mathcal{V}|}) = \text{softmax}(\boldsymbol{l}^{(t)}) = (\frac{\exp(l^{(t)}_1)}{\sum_{j=1}^{|\mathcal{V}|}\exp(l^{(t)}_j)},\dots,\frac{\exp(l^{(t)}_{|\mathcal{V}|})}{\sum_{j=1}^{|\mathcal{V}|}\exp(l^{(t)}_j)}) $ over the vocabulary.
The next token is then sampled from $\boldsymbol{p}^{(t)}$ (for example, multinomial sampling). 

\subsection{LM-based Steganography}
Alice (the sender) wants to communicate a secret message $m_s \sim$ $U(\{0,1\}^l)$ with Bob (the receiver) by embedding it in a natural-language text $t_s$ (a stegotext).
Alice and Bob have agreed on the steganographic embedding function $\mathcal{S}_{\text{emb}}$, the steganographic extracting function $\mathcal{S}_{\text{ext}}$, a language model $\text{LM}$, and LM parameters $\theta$.
These two functions are supposed to be invertible.
Specifically:
\begin{equation}
    \mathcal{S}_{\text{emb}}(f_{\mathrm{LM}}(\cdot;\theta),m_s) = t_s
\end{equation}
\begin{equation}
    \mathcal{S}_{\text{ext}}(f_{\mathrm{LM}}(\cdot;\theta),t_s)=m_s^{\prime}
\end{equation}
where $f_{\mathrm{LM}}(\cdot;\theta)$ denotes that this function is used autoregressively during the embedding and extraction processes and the input context window is dynamically updated.

There have been various methods on how to achieve $\mathcal{S}_{\text{emb}}$ and $\mathcal{S}_{\text{ext}}$, including methods based on block coding~\cite{fang-etal-2017-generating}, Huffman coding~\cite{8470163,dai-cai-2019-towards}, and arithmetic coding~\cite{ziegler-etal-2019-neural,shen-etal-2020-near}, and several much more sophisticated methods aimed at provably secure steganography~\cite{zhang-etal-2021-provably,10.1145/3460120.3484550,witt2023perfectly,10179287,wang2025sparsampefficientprovablysecure}.
Besides, single-step KL divergence matters for the imperceptibility of LM-based steganography (analyzed in Appendix~\ref{sec: imperceptibility_GLS}).

\subsection{Robustness of LM-based Steganography}
External alterations, such as token insertion, deletion, or replacement, in the stegotexts can lead to $m_s \neq m_s^{\prime}$ thus compromising the robustness of steganography. 
However, external alterations have been ignored in most linguistic steganographic methods. The lack of robustness consideration makes them less practical, as reliance on completely unchanged stegotext is difficult to achieve in uncontrolled channels under surveillance. Active attackers can alter the stegotexts, further complicating covert communication.

To mitigate the effect of external alterations, WinStega~\cite{10888944} introduces an adaptive robust enhancement framework, of which the key idea is to limit the adopted context window to the only latest tokens. Therefore, any alteration outside the adopted context window cannot affect inference and extraction. 
However, there are limitations in WinStega:
(1) The original prompt is overly excluded from the context window, although the prompt is initially held by both Alice and Bob and cannot be attacked externally, and the prompt, as initial tokens, matters for the subsequent inference~\cite{xiao2024efficient}.
(2) WinStega introduces an entropy-threshold mechanism that relies on computing the entropy of tokens in the context window, but the computation relies on earlier tokens and inference. Therefore, this mechanism relies on tokens outside the context window, which negatively affects robustness.

Another more recent attempt related to robust linguistic steganography is proposed by~\citet{11023508}. This work focuses on an asymmetric scenario where extraction is not based on model inference, at the tremendous cost of embedding capacity (bits per token $ \in [10^{-3}, 10^{-1}]$ in most cases).
As this work lies outside the mainstream symmetric settings and suffers from impractically limited embedding capacity, it is reasonable to consider it to be not comparable to our ASW.

\section{Anchored Sliding Window (ASW)}
\label{sec: Anchored Sliding Window (ASW)}

In this work, we are inspired by the idea of improving robustness by limiting the context window and the sliding window mechanism~\cite{10888944}.
To preserve semantic information while maintaining robustness, we structure our anchored sliding window (ASW) in which the \textbf{prompt} is anchored as the first segment. We then introduce a \textbf{bridge context} as the second segment, designed to help the model compensate for the excluded tokens. The final segment of the ASW is the \textbf{latest tokens} (see Figure~\ref{fig: ASW_context}).
The design of the bridge context adheres to the following principles:
\begin{itemize}
    \item It is anchored immediately after the prompt.
    \item It is \textbf{independent} of both the content and position of the excluded segment.
\end{itemize}
This independence is critical: any dependency on the excluded segment would render the bridge context sensitive to changes in that segment, thereby undermining robustness. 

Formally, the basic context window for robustness (see Figure~\ref{fig: WinStega_context}) is defined as: $\boldsymbol{s}_{\text{window}}^{(t-1)} = [s^{(\max(-N_p,t-w))},...,s^{(t-1)}]$, where $N_p$ is the prompt length.
In ASW, the context window is: 
\begin{equation}
    \boldsymbol{s}_{\text{window}}^{(t-1)} = \boldsymbol{s}_{\text{prompt}} \Vert \boldsymbol{s}_{\text{bridge}} \Vert \boldsymbol{s}_{\text{latest}}^{(t-1)}
\end{equation}
where $\boldsymbol{s}_{\text{prompt}} = [s^{(-N_p)},...,s^{(-1)}]$, $\boldsymbol{s}_{\text{bridge}}$ is a customizable bridge context, and $\boldsymbol{s}_{\text{latest}}^{(t-1)} = [s^{(\max(0,t-w))},...,s^{(t-1)}]$. Notice that $w$ is the length of the latest tokens, assuming $s^{(t-w)}$ exists within the valid range. Analysis of the robustness of ASW is detailed in Appendix~\ref{appendix:Analysis of Robustness}, suggesting that the number of influenced positions increases  (i.e., robustness decreases) as $w$ increases in most cases.

\begin{figure}[!t]
 \centering
 \includegraphics[width=\columnwidth]{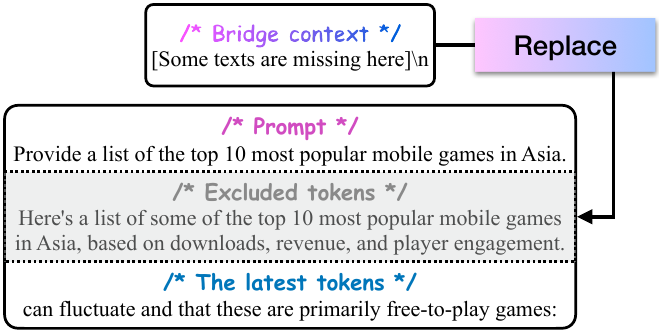} 
 \caption{In ASW, the bridge context replaces the excluded tokens and connects the prompt and the latest tokens. In this example, a hard bridge context is adopted.}
 \label{fig: Hard_Bridge_Context}
\end{figure}

\begin{table*}[!t]
\renewcommand{\arraystretch}{1.0}
\centering
\scalebox{0.88}{
\begin{tabular}{l|l|ccccc}
\toprule[1.0pt]
 \textbf{Context window}  & \textbf{Hard bridge context}   & $w = 10$ & $w = 20$ & $w = 30$ & $w = 40$ & $w = 50$ \\ 
\midrule[1.0pt]
\textbf{Basic} (Figure~\ref{fig: WinStega_context}) & N/A  & \textcolor{red!75!black}{\textbf{3.831}}  & \textcolor{red!75!black}{\textbf{2.054}} &  \textcolor{red!75!black}{\textbf{1.390}}  &  \textcolor{red!75!black}{\textbf{1.023}}  &  \textcolor{red!75!black}{\textbf{0.777}}  \\ \hline
\multirow{7}{*}{\textbf{ASW} (Figure~\ref{fig: ASW_context})} & 0 token  & 2.476  & 1.458   & 1.003   & 0.766   & 0.568   \\ \cline{2-7}
 & Random 5 tokens  & 2.681 &  1.529  &  1.071  & 0.774   &  0.589  \\
 & Random 10 tokens  & 2.705 &  1.580  &  1.094  &  0.775  &  0.592   \\ \cline{2-7}
 & \texttt{``[Some texts are missing here]\textbackslash n''}  & 2.279  & 1.303   &  0.936  &  0.690  & 0.507   \\
 & \texttt{``[previous message removed]\textbackslash n''} & 2.240  & 1.315   &  0.937  & 0.682   & 0.505   \\
 & \texttt{``[CONTEXT TRUNCATED]\textbackslash n''} & 2.201 &  1.268  &  \textcolor{green!50!black}{\textbf{0.904}}  &  \textcolor{green!50!black}{\textbf{0.655}}  & 0.496   \\
 & \texttt{``...\textbackslash n''} & \textcolor{green!50!black}{\textbf{2.195}} &  \textcolor{green!50!black}{\textbf{1.266}}  &  0.908  &  0.664  & \textcolor{green!50!black}{\textbf{0.487}}  \\
\bottomrule[1.0pt]
\end{tabular}}
 \caption{Average KL divergence $D_{\text{KL}}(\boldsymbol{p}_{\text{full}}^{(t)}||\boldsymbol{p}_{\text{window}}^{(t)})$ at each step $t$ under different settings (lower is better). \textcolor{red!75!black}{\textbf{Red}} values indicate the highest (the worst), and \textcolor{green!50!black}{\textbf{green}} values indicate the lowest (the best) in each column.}
 \label{table: KL_hard}
\end{table*}

\begin{figure*}[!t]
 \centering
 \includegraphics[width=\textwidth]{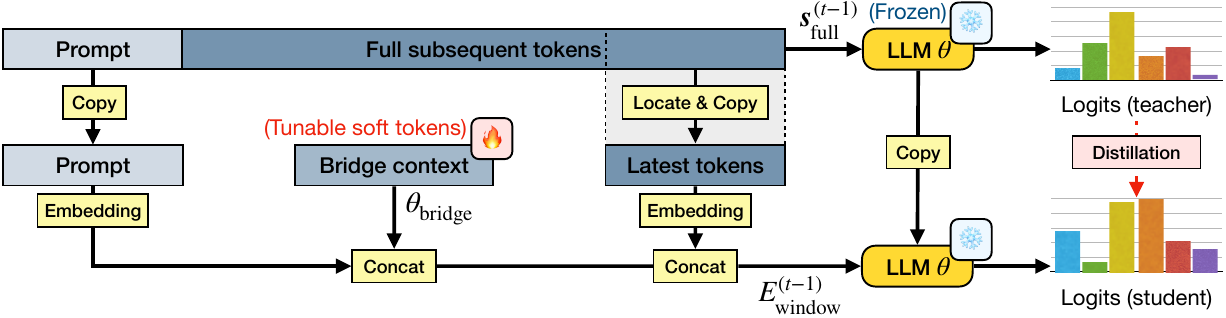} 
 \caption{An overview of our proposed distillation framework. All the parameters of the language model are not tunable, and only the soft bridge context is tunable. The distillation objective is to mitigate the gap between logits based on the full context $\boldsymbol{s}_{\text{full}}^{(t-1)}$ (teacher) and the logits based on a $E_{\text{window}}^{(t-1)}$ (student).}
 \label{fig: Self-disttilation}
\end{figure*}

\section{Intuition: Hard Bridge Context}
\label{sec: Intuition: Hard Bridge Context}
To convey the core intuition behind our work and facilitate the explanation of subsequent methods, we begin with the \textbf{hard} bridge context, where tokens are discretely represented in a vocabulary~\cite{lester-etal-2021-power}.
Similar to prompt engineering, a hard bridge context can be manually designed, for example, using a placeholder like \texttt{``[Some texts are missing here]\textbackslash n''}, to indicate to the language model that part of the input is missing.
In this setting, Figure~\ref{fig: Hard_Bridge_Context} presents an illustrative example of how a full context is processed within our ASW framework at a single generative step.

Next, we conducted preliminary experiments to validate the effectiveness of our intuition and the hard bridge context.
In the context of LM-based steganography, imperceptibility (Appendix~\ref{sec: imperceptibility_GLS}) is closely associated with the the average single-step KL divergence, which serves as a key indicator for quantifying the discrepancy between normal text and text generated under interference.
For now, we set aside the steganographic setting and focus solely on evaluating how limited context windows influence the average single-step KL divergence.

Specifically, the KL divergence at the $t^{\text{th}}$ inference step is defined as:
\begin{align}
    & D_{\text{KL}}(\boldsymbol{p}_{\text{full}}^{(t)}||\boldsymbol{p}_{\text{window}}^{(t)}) \notag \\
    & = \sum_{i=1}^{|\mathcal{V}|}\boldsymbol{p}^{(t)}_{\text{full}}(i) \times \log\left(\frac{\boldsymbol{p}^{(t)}_{\text{full}}(i)}{\boldsymbol{p}^{(t)}_{\text{window}}(i)}\right)
\end{align}
where $\boldsymbol{p}_{\text{full}}^{(t)}$ denotes the model’s output distribution given the full context $\boldsymbol{s}_{\text{full}}^{(t-1)} = [s^{(-N_p)},...,s^{(t-1)}]$ (standard inference),  and $\boldsymbol{p}_{\text{window}}^{(t)}$ denotes the output distribution given a truncated or structured context window, $\boldsymbol{s}_{\text{window}}^{(t-1)}$. These distributions are: 
\begin{equation}
    \boldsymbol{p}_{\text{full}}^{(t)} = \text{softmax}(f_{\text{LM}}(\boldsymbol{s}_{\text{full}}^{(t-1)} ; \theta))
\end{equation}
\begin{equation}
    \boldsymbol{p}_{\text{window}}^{(t)} = \text{softmax}(f_{\text{LM}}(\boldsymbol{s}_{\text{window}}^{(t-1)}; \theta)).
\end{equation}


Table~\ref{table: KL_hard} reports the average KL divergence $D_{\text{KL}}(\boldsymbol{p}_{\text{full}}^{(t)}||\boldsymbol{p}_{\text{window}}^{(t)})$ during inference using WinStega or our ASW, evaluated across $w \in \{10,20,30,40,50\}$. 
The experiments were conducted using the Qwen2.5-7B-Instruct language model~\cite{qwen2.5}, with 500 samples per setting.
Each sample was derived from a randomly selected question-answer (QA) pair from the InstructionWild dataset~\cite{instructionwild}, with a maximum sequence length of 512 tokens.
From the table, we observe the following key findings:

\textbf{(1) Prompt matters:} Compared to the basic context window, ASW without a hard bridge context (i.e., with 0 tokens) shows a substantial reduction in divergence. This supports the idea of \textit{attention sinks}~\cite{xiao2024efficient}.

\textbf{(2) Design of bridge context matters:} Compared to using 0 tokens or randomly selected 5 or 10 tokens as the bridge context, a human-readable and instructive bridge context further reduces KL divergence. This suggests that the language model is able to utilize the semantic cues to better infer or \textit{imagine} the excluded content.

\section{Tunable Soft Bridge Context via Self-Distillation}
\label{sec: Tunable Soft Bridge Context via Self-Distillation}
Since the design of the bridge context plays an important role in narrowing the gap between inference based on the full context and that based on  $\boldsymbol{s}_{\text{window}}$,  we are motivated to explore the  \textit{optimal} bridge context for our ASW framework.
However, the hard bridge context lacks a clear design principle, making systematic optimization difficult. To address this, we propose using a tunable \textbf{soft} bridge context to connect the prompt and the latest tokens, which can be trained with an explicit objective.

\subsection{Soft Context Window}

This approach aligns with the concept of \textit{prompt tuning}~\cite{lester-etal-2021-power}, in which the soft bridge context is no longer composed of discrete tokens of the vocabulary. 
Instead, given a bridge context length $l_{\text{bridge}}$, the soft bridge context $\boldsymbol{s}_{\text{bridge}}$ is represented as a matrix $E_{\text{bridge}} \in \mathbb{R}^{l_{\text{bridge}} \times e}$, where $e$ is the dimension of the embedding space.
In our ASW framework, when a soft bridge context is employed, the context window is constructed at the embedding level (at time $t-1$) as $E_{\text{window}}^{(t-1)} = [E_{\text{prompt}}; E_{\text{bridge}}; E_{\text{latest}}^{(t-1)}]$.
$E_{\text{bridge}}$ is also denoted as tunable parameters $\theta_{\text{bridge}}$, and $E_{\text{window}}^{(t-1)}$ is also denoted as $E_{\text{window}}^{(t-1)}({\theta_{\text{bridge}}})$, that is:
\begin{equation}
E_{\text{window}}^{(t-1)}({\theta_{\text{bridge}}}) = [E_{\text{prompt}}; \theta_{\text{bridge}}; E_{\text{latest}}^{(t-1)}].
\end{equation}

\subsection{Self-Distillation}
In this section, we present the training strategy for the soft bridge context.
Figure~\ref{fig: Self-disttilation} provides an overview of the framework used to tune the soft bridge context. This process can be viewed as a form of distillation, specifically a variant of \textbf{prompt distillation}~\cite{10.1145/3583780.3615017}, where one entity (the student) learns to mimic another entity (the teacher). Specifically:
\begin{itemize}
    \item The \textbf{teacher output} is the logit vector for the last token, inferred from the full context.
    \item The \textbf{student output} is the logit vector for the last token, inferred from $E_{\text{window}}^{(t-1)}({\theta_{\text{bridge}}})$.
\end{itemize}
Since both the teacher and student outputs are generated by the same frozen language model, this setup constitutes a form of \textbf{self-distillation}.

We propose a loss function which is based on the single-step forward KL divergence (for each step), because forward KL divergence is a significant indicator of imperceptibility (Appendix~\ref{sec: imperceptibility_GLS}) and also complies with common methodology of knowledge distillation~\cite{kim2016sequence, sanh2020distilbertdistilledversionbert, song2020lightpafftwostagedistillationframework,Li_2024_CVPR}. Specifically, the loss function of our self-distillation strategy is:
\begin{align}
\label{eq: Forward KL}
    & \mathcal{L}_{\text{forward}}(\theta_{\text{bridge}}) \notag  =  \frac{1}{T}\sum_{t=0}^{T-1}D_{\text{KL}}\\ &(\text{softmax}(\boldsymbol{l}^{(t)}_{\text{teacher}})  \Vert  \text{softmax}(\boldsymbol{l}^{(t)}_{\text{student}}(\theta_{\text{bridge}})))
\end{align}
where $T$ is the length of the non-prompt part, $\boldsymbol{l}^{(t)}_{\text{teacher}} = f_{\text{LM}}(\boldsymbol{s}_{\text{full}}^{(t-1)}; \theta)$ and $\boldsymbol{l}^{(t)}_{\text{student}}(\theta_{\text{bridge}}) = f_{\text{LM\_emb}}(E_{\text{window}}^{(t-1)}(\theta_{\text{bridge}}); \theta)$. $f_{\text{LM\_emb}}(\cdot)$ is a variant of $f_{\text{LM}}(\cdot)$ to process embeddings instead of tokens.

In addition to using the forward KL divergence as the loss function (Equation~\ref{eq: Forward KL}), we also explore the use of reverse KL divergence for distillation~\cite{gu2024minillm} in practice, specifically: 
\begin{align}
\label{eq: reverse KL}
 & \mathcal{L}_{\text{reverse}}(\theta_{\text{bridge}}) \notag   =  \frac{1}{T}\sum_{t=0}^{T-1}D_{\text{KL}} \\ & ( \text{softmax}(\boldsymbol{l}^{(t)}_{\text{student}}(\theta_{\text{bridge}})) \Vert  \text{softmax}(\boldsymbol{l}^{(t)}_{\text{teacher}})).
\end{align}

\begin{algorithm}[!t]
\small
\caption{Steganographic embedding in ASW (with a soft bridge context)}\label{algorithm: Steganographic embedding in ASW (with soft bridge context)} 
\textbf{Input:}\\
Prompt, $\boldsymbol{s}_{\text{prompt}} = [s^{(-N_p)},\dots,s^{(-1)}]$ \\
Secret message, $m_s$  \\
Parameters of the language model, $\theta$ \\
Parameters of a soft bridge context after training, $\theta_{\text{bridge}}$ \\
Length of the latest tokens, $w$ \\
{\textbf{Output:}}\\
Steganographic text, $t_s$\\

\begin{algorithmic}[1]
\STATE $\boldsymbol{s}_{\text{generated}} \leftarrow []$;
\FOR{$t = 0,1,\dots$}
    \STATE $\boldsymbol{s}_{\text{latest}} \leftarrow \boldsymbol{s}_{\text{generated}}[-w:]$;
    \STATE Get the embedding matrices $E_{\text{prompt}}$ and $E_{\text{latest}}$ from $\boldsymbol{s}_{\text{prompt}}$ and $\boldsymbol{s}_{\text{latest}}$;
    \STATE $E_{\text{window}} \leftarrow [E_{\text{prompt}}; \theta_{\text{bridge}} ;E_{\text{latest}}]$;
    \STATE $\boldsymbol{p}^{(t)}_{\text{window}} \leftarrow \text{softmax}(f_{\text{LM\_emb}}(E_{\text{window}};\theta))$; 
    \STATE Use the steganographic embedding algorithm,  $\boldsymbol{p}^{(t)}_{\text{window}}$, and $m_s$ to generate the next token $s^{(t)}$;
    \STATE $\boldsymbol{s}_{\text{generated}} \leftarrow \boldsymbol{s}_{\text{generated}} \Vert s^{(t)}$;
\ENDFOR

\STATE Detokenize $\boldsymbol{s}_{\text{generated}}$ to $t_s$;
\RETURN $t_s$

\end{algorithmic}

\end{algorithm}

\section{Steganography in ASW Framework}

In this section, we describe how linguistic steganography is implemented within the ASW framework.
Algorithm~\ref{algorithm: Steganographic embedding in ASW (with soft bridge context)} and Algorithm~\ref{algorithm: Steganographic extraction in ASW (with soft bridge context)} detail the embedding and extraction procedures, respectively, for steganographic communication using ASW framework with a trained \textbf{soft} bridge context.

Compared to conventional steganographic methods, the ASW framework involves a reorganization of the context window, specifically, in Lines 3--5 of the embedding algorithm and Lines 4--6 of the extraction algorithm.
Note that the specific steganographic method (e.g., arithmetic coding~\cite{ziegler-etal-2019-neural}) is abstracted away in this presentation and is represented in Line 7 of the embedding algorithm and Line 8 of the extraction algorithm without further elaboration.

For completeness, the embedding and extraction procedures using a \textbf{hard} bridge context are provided in Algorithms~\ref{algorithm: Steganographic embedding in ASW (with hard bridge context)} and~\ref{algorithm: Steganographic extraction in ASW (with hard bridge context)} in the Appendix.

\begin{algorithm}[!t]
\small
\caption{Steganographic extraction in ASW (with a soft bridge context)}\label{algorithm: Steganographic extraction in ASW (with soft bridge context)} 
\textbf{Input:}\\
Prompt, $\boldsymbol{s}_{\text{prompt}} = [s^{(-N_p)},\dots,s^{(-1)}]$ \\
Steganographic text, $t_s$\\
Parameters of the language model, $\theta$ \\
Parameters of a soft bridge context after training, $\theta_{\text{bridge}}$ \\
Length of the latest tokens, $w$ \\
{\textbf{Output:}}\\
Secret message, $m_s^{\prime}$  \\

\begin{algorithmic}[1]
\STATE $m_s^{\prime} \leftarrow \emptyset$;
\STATE Tokenize $t_s$ to $\boldsymbol{s}_{\text{generation}} = [s^{(0)},s^{(1)},\dots]$;
\FOR{$t = 0,1,\dots$}
    \STATE $\boldsymbol{s}_{\text{latest}} \leftarrow \boldsymbol{s}_{\text{generated}}[t-w:t]$;
    \STATE Get the embedding matrices $E_{\text{prompt}}$ and $E_{\text{latest}}$ from $\boldsymbol{s}_{\text{prompt}}$ and $\boldsymbol{s}_{\text{latest}}$;
    \STATE $E_{\text{window}} \leftarrow [E_{\text{prompt}}; \theta_{\text{bridge}} ;E_{\text{latest}}]$;
    \STATE $\boldsymbol{p}^{(t)}_{\text{window}} \leftarrow \text{softmax}(f_{\text{LM\_emb}}(E_{\text{window}};\theta))$; 
    \STATE Use the steganographic extraction algorithm,  $\boldsymbol{p}^{(t)}_{\text{window}}$, and $s^{(t)}$ to update $m_s^{\prime}$;
\ENDFOR
\RETURN $m_s^{\prime}$

\end{algorithmic}

\end{algorithm}

\begin{table*}[!t]
\renewcommand{\arraystretch}{1.0}
\centering
\scalebox{0.68}{
\begin{tabular}{l|l|cccc|c|cc}
\toprule[1pt]
\multirow{2}{*}{\textbf{Method}} & \multirow{2}{*}{\textbf{Bridge context}} & \multicolumn{4}{c|}{\textbf{Text quality}} & \textbf{Imperceptibility}      & \multicolumn{2}{c}{\textbf{Efficiency}} \\ \cline{3-9}
                        &                                    & $\Delta$PPL$\downarrow$  & BLEU$\uparrow$  & ROUGE-L$\uparrow$ & BERTScore$\uparrow$ & Steganalysis ACC$\downarrow$ & Capacity$\uparrow$     & Time (s)$\downarrow$     \\ \hline \rowcolor{gray!20} 
\textbf{Full context}            & N/A    &  1.090   &  0.405     &  0.257     &    0.856       &       0.560                &   1.024      &   115.43   \\  \hline
\textbf{WinStega}                & N/A                                &  \textcolor{red!75!black}{\textbf{31.486}}    &  \textcolor{red!75!black}{\textbf{0.082}}  &    \textcolor{red!75!black}{\textbf{0.098}}   &    \textcolor{red!75!black}{\textbf{0.699}}       &       \textcolor{red!75!black}{\textbf{0.955}}                &  1.995       &       22.250               \\ \hline
  \textbf{ASW}  & 0 token       &   5.157   &   0.220    &  0.190     &     0.806      &        0.885               &    0.966     &      29.702                \\ \hline
                 \multirow{2}{*}{\textbf{ASW (hard)}}       & \texttt{``[CONTEXT TRUNCATED]\textbackslash n''}                             &  3.246    &   0.242    &   0.195    &    0.816       &    0.830                   &   1.012      &   30.868                   \\
                        & \texttt{``...\textbackslash n''}                             &  4.560    &   0.248    &   0.195    &     0.816      &         0.815              &     1.035    &       30.505               \\ \cline{1-9}
\multirow{3}{*}{\textbf{ASW (soft)}}   & Random untrained  $\theta_{\text{bridge}}$                           &  4.066    &   0.200    &   0.185    &    0.797       &    0.910      &     0.980    &    31.819                  \\
                        & $\theta_{\text{bridge}}$ trained with forward KL                       & \textcolor{green!50!black}{\textbf{0.201}}    &   \textcolor{green!50!black}{\textbf{0.276}}    &  \textcolor{green!50!black}{\textbf{0.200}}     & \textcolor{green!50!black}{\textbf{0.828}}           &     \textcolor{green!50!black}{\textbf{0.745}}                  &   1.606      &     32.594                 \\
                        & $\theta_{\text{bridge}}$ trained with reverse KL                      &  4.336    &   0.267    &    \textcolor{green!50!black}{\textbf{0.200}}    &    0.825       &       0.770                &    1.124     & 31.499   \\ \bottomrule[1pt]                  
\end{tabular}}
 \caption{Average results across various metrics under different methods, where the setting is Qwen2.5-7B-Instruct, $l_{\text{bridge}} = 8$, and $w = 10$. \textcolor{red!75!black}{\textbf{Red}} values indicate the worst, and \textcolor{green!50!black}{\textbf{green}} values indicate the best in each column.}
 \label{table: Main result}
\end{table*}

\section{Experiments}
\label{sec: Experiments}
\subsection{General Setup}
\label{sec: General Setup}
To perform linguistic steganography, we adopted arithmetic coding~\cite{ziegler-etal-2019-neural} as the steganographic method, as it offers a strong trade-off between efficiency and imperceptibility.
For our ASW framework, the bridge context settings are as follows:
\textbf{(1) Hard bridge context:} Based on the results in Table~\ref{table: KL_hard}, we selected \texttt{``[CONTEXT TRUNCATED]\textbackslash n''} and  \texttt{``...\textbackslash n''} as two effective hard bridge contexts $\boldsymbol{s}_{\text{bridge}}$ in subsequent experiments.
\textbf{(2) Soft bridge context:} To train a soft context bridge $\theta_{\text{bridge}}$ via self-distillation, we randomly selected 12,500 training and 1,000 validation samples from the InstructionWild dataset~\cite{instructionwild}. Training is conducted for 10 epochs, and the soft bridge context yielding the lowest validation loss was selected. We used the AdamW optimizer~\cite{loshchilov2018decoupled} with a learning rate of $1\times 10^{-3}$. 

Specifically, we used top-$p$ ($p=1.0$) sampling and temperature $\tau = 1.0$.
Generation ends when \texttt{<eos>} is generated (the maximum length is 512), and an infinite and random sequence of bits was provided for embedding.


\subsection{Main Experiments}
\label{sec: Main Experiments}
In this section, we conducted experiments to show the superiority of our ASW framework compared to the baseline method, WinStega~\cite{10888944}.
Specifically, we used the Qwen2.5-7B-Instruct model~\cite{qwen2.5}, with the length of the soft bridge context set to $l_{\text{bridge}} = 8$ and the length of the latest tokens $w = 10$ (referring to WinStega). For evaluation, 500 questions from the InstructionWild dataset were used as prompts to generate 500 samples for each experimental group.

\subsubsection{Metrics}
Reference texts are the non-steganographic texts generated by multinomial sampling based on full-context inference prompted by the same 500 questions for evaluation.
The metrics for text quality are based on similarity between reference texts and stegotexts, including (1) the absolute difference in perplexity $\Delta \text{PPL}$~\cite{jelinek1977perplexity} (the average perplexity of the reference texts is 15.148 in this setting). (2) BLEU (2-gram)~\cite{lin-och-2004-orange}, (3) ROUGE-L~\cite{lin-2004-rouge}, and (4) BERTScore~\cite{bert-score}. The metric for imperceptibility is based on steganalysis accuracy, of which the experimental details are shown in Appendix~\ref{sec: steganalysis}. Besides, the metrics for efficiency include (1) embedding capacity (embedded bits per token), and (2) embedding time for generating a stegotext (seconds).

\begin{figure*}[!t]
    \centering
    \begin{subfigure}[b]{\textwidth} 
        \centering
        \includegraphics[width=\textwidth]{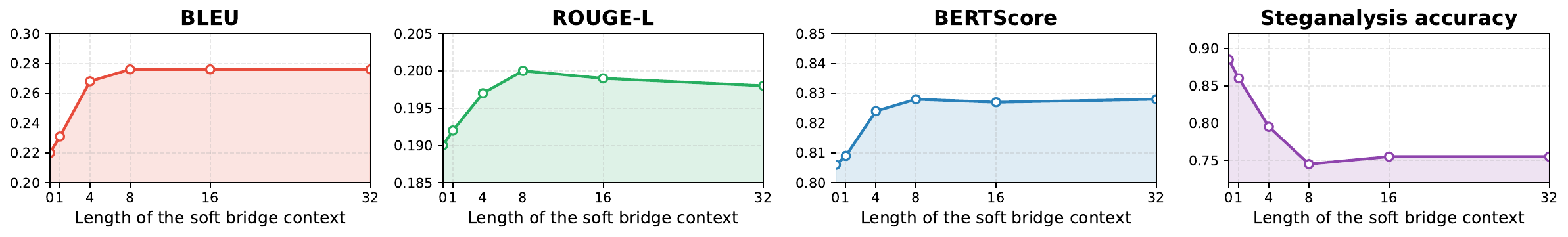} 
        \caption{$\theta_{\text{bridge}}$ trained with forward KL.}
        \label{subfig: Bridge_length_forward_quality} 
    \end{subfigure}
    \hfill 
    \begin{subfigure}[b]{\textwidth}
        \centering
        \includegraphics[width=\textwidth]{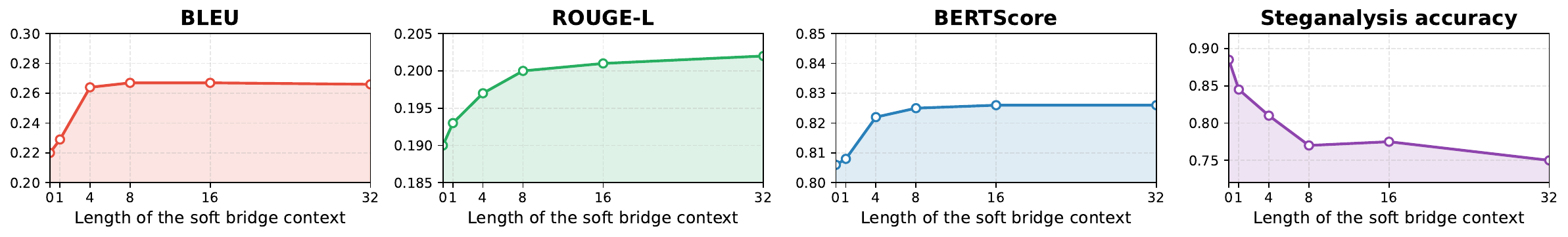} 
        \caption{$\theta_{\text{bridge}}$ trained with reverse KL.}
        \label{subfig: Bridge_length_reverse_quality} 
    \end{subfigure}
    \caption{Average values of various metrics when the length of the soft bridge context $l_{\text{bridge}}$ varies.}
    \label{fig: Bridge_length_quality} 
\end{figure*}

\begin{figure*}[!t]
 \centering
 \includegraphics[width=\textwidth]{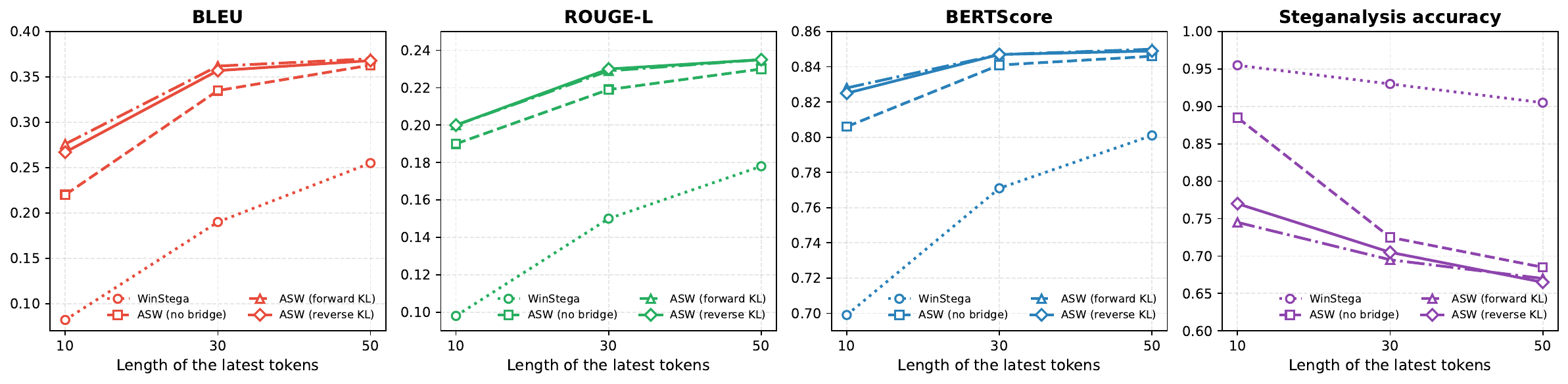} 
 \caption{Average values of various metrics under WinStega or our proposed ASW framework equipped with different strategies when the length of the latest tokens $w$ varies.}
 \label{fig: Latest_length}
\end{figure*}

\subsubsection{Results and Analysis}
Table~\ref{table: Main result} reports the average results across various metrics using different methods. In addition to the baseline method (WinStega) and our ASW variants, the table also includes steganographic performances under the \textbf{full context} setting, which corresponds to the conventional robustness-unaware linguistic steganography for reference (see Figure~\ref{fig: full_context}).
From the table, the key findings are as follows:

\textit{1)} Our ASW substantially outperforms WinStega in text quality and imperceptibility. Comparing their best performances, ASW lowers $\Delta$PPL by 99.36\% (31.486 $\rightarrow$ 0.201), improves BLEU by 236.59\% (0.082 $\rightarrow$ 0.276), improves ROUGE-L by 104.08\% (0.098 $\rightarrow$ 0.200), improves BERTScore by 18.45\% (0.699 $\rightarrow$ 0.828), and improves the anti-detection capacity, i.e., lowering steganalysis accuracy by 21.99\% (0.955 $\rightarrow$ 0.745).

\textit{2)} Considering the embedding capacity of 1.024 under the full-context setting as the standard capacity, most of our ASW capacities are close to this value, whereas WinStega achieves a much higher embedding capacity.
One reason is that WinStega uses only the most recent $w$ tokens, which increases the freedom and entropy of generation but at the cost of less reliable semantic information.

\textit{3)} The runtime for generating a stegotext is highly correlated with the number of tokens in the context window. WinStega achieves the highest speed, while full-context steganography requires significantly more time than WinStega and ASW.

\begin{table}[!t]
\centering
\scalebox{0.85}{
\begin{tabular}{l|l|ccc}
\toprule[1pt]
                      \multicolumn{2}{c|}{Method}           & \multicolumn{1}{c}{$m = 1$} & \multicolumn{1}{c}{$m = 2$} & \multicolumn{1}{c}{$m = 3$} \\ \hline \rowcolor{gray!20} 
\multicolumn{2}{c|}{Full context} & 49.96\% & 34.59\% & 24.61\% \\ \hline
\multirow{2}{*}{$w = 10$} & WinStega &  96.35\%      &  92.93\%    &  89.98\%  \\
                        & ASW      & \textbf{97.95\%}      & \textbf{96.21\%}      & \textbf{94.12\%}         \\ \hline

\multirow{2}{*}{$w = 30$} & WinStega &            89.97\%            &        81.01\%                &      72.90\%               \\
& ASW      &    \textbf{94.28\%}   &   \textbf{88.95\%}    &   \textbf{83.67\%}      \\ \hline
\multirow{2}{*}{$w = 50$} & WinStega &            84.40\%               &     72.95\%                      &       58.13\%                    \\
                        & ASW      &     \textbf{90.03\%}        &      \textbf{82.55\%}   &   \textbf{74.90\%}    \\ \bottomrule[1pt]                   
\end{tabular}}
 \caption{The ratio of positions with unaffected inference when the number $m$ of modified tokens and the latest token length $w$ vary.}
 \label{table: robustness}
\end{table}

\subsection{Robustness Scenarios}
\label{sec: Robustness Scenarios}
We simulated real-world adversarial attacks by randomly modifying $m$ tokens in the generated stegotexts, following the setup (except $w$) in Section~\ref{sec: Main Experiments}. Since steganographic encoding methods can differ, we focused on a more generalizable factor: whether inference is affected. Accordingly, we used the ratio of positions with \textbf{unaffected} inference to indicate robustness. Table~\ref{table: robustness} lists these ratios under three setups: robustness-unaware full context, WinStega and our ASW (with a hard bridge context).\footnote{The robustness of ASW is independent of whether the bridge context is hard or soft, and depends only on the length of the latest tokens (see Appendix~\ref{appendix:Analysis of Robustness}).}
ASW outperforms WinStega in robustness across different $(w,m)$ pairs, since WinStega relies on tokens outside the context window to calculate entropy, thereby compromising robustness. Moreover, robustness decreases as $w$ increases, consistent with Proposition~\ref{proposition: E(w)} in Appendix~\ref{appendix:Analysis of Robustness}.

The rationale for using the ratio of positions with unaffected inference as an indicator of robustness is provided in Appendix~\ref{sec: Supplementary Robustness Scenarios}, which also includes two additional scenarios: deletion and insertion.

\subsection{Effect of Hyperparameter Variants}
We studied the effects of the length of the soft bridge context $l_{\text{bridge}}$ and the length of the latest tokens $w$. The setup follows Section~\ref{sec: Main Experiments}. 

Figure~\ref{fig: Bridge_length_quality} shows the variation of several metrics as $l_{\text{bridge}}$ changes ($w = 10$). The results indicate that when $l_{\text{bridge}} \geq 8$, further increasing $l_{\text{bridge}}$ provides little improvement in text quality or imperceptibility, while substantially increasing computational overhead (see raw data in Table~\ref{table: Effect of l} in the Appendix). It offers rationality to set $l_{\text{bridge}} = 8$ in Section~\ref{sec: Main Experiments}.

Figure~\ref{fig: Latest_length} shows how several metrics vary with $w$ when $l_{\text{bridge}} = 8$. Overall, the performances improve as $w$ increases, and our ASW consistently outperforms WinStega across different $w$. The raw data are reported in Table~\ref{table: Effect of w} in the Appendix.

\subsection{Effect of Model and Data Variants}
We studied the impact of the model scale and dataset choice, to validate the generalizability of our ASW further. 
In addition to Qwen2.5-7B-Instruct, the model adopted in Section~\ref{sec: Main Experiments}, Qwen2.5-3B-Instruct and Qwen2.5-14B-Instruct were examined. 
Beyond the InstructionWild dataset we performed evaluations on two out-domain datasets, databricks-dolly-15k~\cite{DatabricksBlog2023DollyV2} and Super-NaturalInstructions~\cite{wang-etal-2022-super}, to validate the \textit{generalizability} of ASW. 
Details are provided in Appendix~\ref{app: Effect of Model and Data Variants}.
Our experiments revealed a desirable trend that \textit{ASW becomes increasingly beneficial as the scale of the adopted language model increases}.


\section{Conclusion}
In this paper, we focus on both imperceptibility and robustness in linguistic steganography. 
We introduce the ASW framework, in which the context window is composed of the prompt, the bridge context, and the latest tokens. 
The bridge context is first instantiated as hard tokens and is extended to a tunable soft bridge context optimized via self-distillation, which mitigates the gap between full-context and ASW-based inference.
Empirical results demonstrate that ASW consistently outperforms the baseline method in text quality, imperceptibility, and robustness, across diverse hyperparameters, models, and datasets.


\section*{Limitations}
Our ASW framework with a soft bridge context trained using forward KL divergence achieves a noticeably higher embedding capacity than the other ASW variants, while showing no compromise in performance on the other metrics. The reason behind this phenomenon needs further exploration.

Due to limitations in computational resources and time, the experiments on hyperparameter variants, model variants, and dataset variants were conducted separately from the main experiments (Section~\ref{sec: Main Experiments}). Each variant was studied in isolation, rather than being combined with others, which may have overlooked some potential interaction effects.

\section*{Ethical Considerations}
Steganography can be used for privacy preservation, but it also carries risks of misuse for  malicious covert communication. Our work is intended solely for academic research, with experiments conducted on public datasets. We highlight that it should be used only in lawful and ethical contexts.

\section*{Acknowledgments}
We express our gratitude to the anonymous reviewers for their valuable and insightful comments.
This work was supported by JST SPRING, Grant Number JPMJSP2110.

\bibliography{custom}

\appendix

\section{Imperceptibility of LM-based Steganography}
\label{sec: imperceptibility_GLS}

Following the previous formulation~\cite{dai-cai-2019-towards, shen-etal-2020-near}, statistical imperceptibility refers to the similarity between the true language model $\text{LM}^t$ in the monitored channel and $\text{LM}^s$ which is the language model $\text{LM}$ integrated with steganographic algorithms. Specifically, the total variation distance (TVD) is used to measure statistical imperceptibility. Consider the TVD between $\text{LM}^t$ and $\text{LM}^s$, i.e. $d(\text{LM}^t, \text{LM}^s)$, by triangle inequality: $d(\text{LM}^t, \text{LM}^s) \leq d(\text{LM}^t, \text{LM}) + d(\text{LM}, \text{LM}^s)$.
As $d(\text{LM}^t, \text{LM})$ is a criterion to measure the original language model, which is limited by the research on language models. Thus, $d(\text{LM}, \text{LM}^s)$ is the main focus of linguistic steganography.

According to Pinsker’s inequality~\cite{1201071} and additivity of KL divergence, $d(\text{LM}, \text{LM}^s)$ can be further decomposed in each step, that is:\footnote{Some derivation is skipped, as details are verified in~\citet{dai-cai-2019-towards, shen-etal-2020-near, 1201071}.}
\begin{equation}\label{eq: KLD}
    d(\text{LM}, \text{LM}^s) \leq \sqrt{\frac{\ln{2}}{2}\sum_{t}D_{\text{KL}}(\boldsymbol{p}^{(t)}||\boldsymbol{\hat{p}}^{(t)})}
\end{equation}
where $\boldsymbol{p}^{(t)}$ is the original probability distribution at $t^{th}$ step, and $\boldsymbol{\hat{p}}^{(t)}$ is transformed from $\boldsymbol{p}^{(t)}$ via sampling and encoding.
Hence,  linguistic steganography could aim to minimize $D_{\text{KL}}(\boldsymbol{p}^{(t)}||\boldsymbol{\hat{p}}^{(t)})$, in order to obtain relative near-imperceptibility.

\begin{table*}[!t]
\renewcommand{\arraystretch}{1.0}
\centering
\scalebox{0.68}{
\begin{tabular}{l|l|cccc|c|cc}
\toprule[1pt]
\multirow{2}{*}{\textbf{Method}} & \multirow{2}{*}{\textbf{Bridge context}} & \multicolumn{4}{c|}{\textbf{Text quality}} & \textbf{Imperceptibility}      & \multicolumn{2}{c}{\textbf{Efficiency}} \\ \cline{3-9}
                        &                                    & $\Delta$PPL$\downarrow$  & BLEU$\uparrow$  & ROUGE-L$\uparrow$ & BERTScore$\uparrow$ & Steganalysis ACC$\downarrow$ & Capacity$\uparrow$     & Time (s)$\downarrow$     \\ \hline \rowcolor{blue!8} 
\multicolumn{9}{c}{Qwen2.5-3B-Instruct} \\  \hline \rowcolor{gray!20} 
\textbf{Full context}            & N/A    &  13.316   &   0.345   &  0.213   &   0.839       &    0.645          &    1.604  & 55.951   \\  \hline 
\textbf{WinStega}                & N/A                                                         &  25.743    & 0.083   &  0.099   &   0.701       &         0.945            &     2.588   &    13.121                      \\ \hline
 \textbf{ASW}   & 0 token          &  16.369  &  0.216        &  0.184    &     0.807      &       0.885       &  1.431      &     14.367                \\ \hline
             \multirow{2}{*}{\textbf{ASW (hard)}}           & \texttt{``[CONTEXT TRUNCATED]\textbackslash n''}                             &   13.488   &   0.220    &   0.184    &     0.808     &          0.855          &    1.710   &  14.831            \\
                        & \texttt{``...\textbackslash n''}                         &  12.693  &  0.231   &  0.186   &  0.812     &                    0.850   &     1.601    &   14.528                  \\ \cline{1-9}
\multirow{3}{*}{\textbf{ASW (soft)}}   & Random untrained  $\theta_{\text{bridge}}$         &  12.516   &  0.209   &  0.183   &    0.804     &    0.900    &    1.348     &     15.286             \\
                        & $\theta_{\text{bridge}}$ trained with forward KL        &  13.046   &  0.253   &  0.190   &    0.820     &          0.785     &     1.984   &   15.360                      \\
                        & $\theta_{\text{bridge}}$ trained with reverse KL                      &  12.337   &  0.251   &  0.191   &  0.818       &    0.795  &    1.635     & 14.998  \\    \hline \rowcolor{blue!8} 
\multicolumn{9}{c}{Qwen2.5-7B-Instruct} \\  \hline \rowcolor{gray!20} 
\textbf{Full context}            & N/A    &  1.090   &  0.405     &  0.257     &    0.856       &       0.560                &   1.024      &   115.43   \\  \hline 
\textbf{WinStega}                & N/A                                &  31.486    &  0.082  &    0.098   &    0.699       &       0.955               &  1.995       &       22.250               \\ \hline
 \textbf{ASW}   & 0 token                            &   5.157   &   0.220    &  0.190     &     0.806      &        0.885               &    0.966     &      29.702                \\ \hline
          \multirow{2}{*}{\textbf{ASW (hard)}}               & \texttt{``[CONTEXT TRUNCATED]\textbackslash n''}                             &  3.246    &   0.242    &   0.195    &    0.816       &    0.830                   &   1.012      &   30.868                   \\
                        & \texttt{``...\textbackslash n''}                             &  4.560    &   0.248    &   0.195    &     0.816      &         0.815              &     1.035    &       30.505               \\ \cline{1-9}
\multirow{3}{*}{\textbf{ASW (soft)}}   & Random untrained  $\theta_{\text{bridge}}$                           &  4.066    &   0.200    &   0.185    &    0.797       &    0.910      &     0.980    &    31.819                  \\
                        & $\theta_{\text{bridge}}$ trained with forward KL                       & 0.201    & 0.276    &  0.200     & 0.828           &     0.745                  &   1.606      &     32.594                 \\
                        & $\theta_{\text{bridge}}$ trained with reverse KL                      &  4.336    &   0.267    &    0.200    &    0.825       &       0.770                &    1.124     & 31.499   \\ \hline \rowcolor{blue!8} 
\multicolumn{9}{c}{Qwen2.5-14B-Instruct} \\  \hline \rowcolor{gray!20} 
\textbf{Full context}            & N/A    &  0.612   &  0.389    &  0.257   &   0.853       &        0.590      &   1.135   & 244.337  \\  \hline 
\textbf{WinStega}                & N/A        &   39.483  &  0.082  &  0.099   &   0.698       &          0.930           &    2.540    &     51.911                \\ \hline
 \textbf{ASW} & 0 token        &  4.022    &    0.225   &  0.193    &    0.811       &      0.900        &     1.073   &       71.789              \\ \hline
          \multirow{2}{*}{\textbf{ASW (hard)}}                & \texttt{``[CONTEXT TRUNCATED]\textbackslash n''}                             &   0.871   &   0.284    &   0.198    &      0.830    &           0.740         &     1.529  &   73.446           \\
                        & \texttt{``...\textbackslash n''}          &  3.645   &   0.239  &  0.195   &   0.816      &      0.820                 &   1.095      &     72.203                \\ \cline{1-9}
\multirow{3}{*}{\textbf{ASW (soft)}}   & Random untrained  $\theta_{\text{bridge}}$         &  3.948   &  0.209   &  0.189   &   0.801      &          0.915             &   1.029      &  74.604                \\
                        & $\theta_{\text{bridge}}$ trained with forward KL      &  0.761   &  0.288   &  0.202   &   0.833      &   0.725                    &     1.536    &    74.986              \\
                        & $\theta_{\text{bridge}}$ trained with reverse KL      &  1.903   &  0.286   &   0.203  &    0.831     &                     0.730  &    1.198     & 74.317  \\  \bottomrule[1pt]                  
\end{tabular}}
 \caption{Average results across various metrics under different language models and different methods, where the setting is $l_{\text{bridge}} = 8$, and $w = 10$.}
 \label{table: Effect of model}
\end{table*}

\begin{figure*}[!t]
 \centering
 \includegraphics[width=\textwidth]{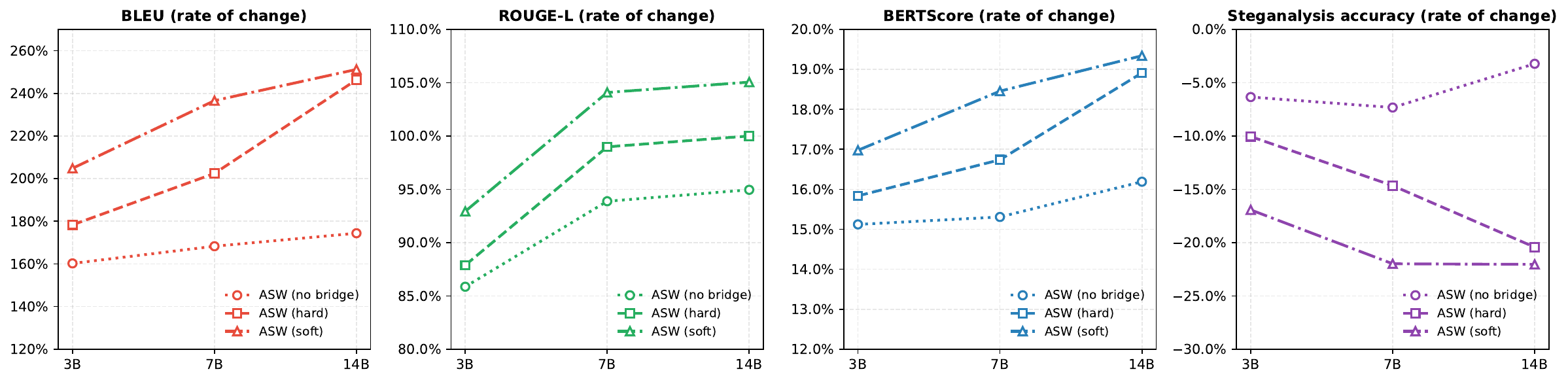} 
 \caption{Compared to WinStega, rates of change of various metrics when the scale of the language model varies.}
 \label{fig: Scale_extent}
\end{figure*}

\begin{table*}[!t]
\renewcommand{\arraystretch}{1.0}
\centering
\scalebox{0.68}{
\begin{tabular}{l|l|cccc|c|cc}
\toprule[1pt]
\multirow{2}{*}{\textbf{Method}} & \multirow{2}{*}{\textbf{Bridge context}} & \multicolumn{4}{c|}{\textbf{Text quality}} & \textbf{Imperceptibility}      & \multicolumn{2}{c}{\textbf{Efficiency}} \\ \cline{3-9}
                        &                                    & $\Delta$PPL$\downarrow$  & BLEU$\uparrow$  & ROUGE-L$\uparrow$ & BERTScore$\uparrow$ & Steganalysis ACC$\downarrow$ & Capacity$\uparrow$     & Time (s)$\downarrow$     \\     \hline \rowcolor{blue!8} 
\multicolumn{9}{c}{InstructionWild~\cite{instructionwild}} \\  \hline \rowcolor{gray!20} 
\textbf{Full context}            & N/A    &  1.090   &  0.405     &  0.257     &    0.856       &       0.560                &   1.024      &   115.43   \\  \hline 
\textbf{WinStega}                & N/A                                &  31.486    &  0.082  &    0.098   &    0.699       &       0.955               &  1.995       &       22.250               \\ \hline
  \textbf{ASW} & 0 token                            &   5.157   &   0.220    &  0.190     &     0.806      &        0.885               &    0.966     &      29.702                \\ \hline
        \multirow{2}{*}{\textbf{ASW (hard)}}                 & \texttt{``[CONTEXT TRUNCATED]\textbackslash n''}                             &  3.246    &   0.242    &   0.195    &    0.816       &    0.830                   &   1.012      &   30.868                   \\
                        & \texttt{``...\textbackslash n''}                             &  4.560    &   0.248    &   0.195    &     0.816      &         0.815              &     1.035    &       30.505               \\ \cline{1-9}
\multirow{3}{*}{\textbf{ASW (soft)}}   & Random untrained  $\theta_{\text{bridge}}$                           &  4.066    &   0.200    &   0.185    &    0.797       &    0.910      &     0.980    &    31.819                  \\
                        & $\theta_{\text{bridge}}$ trained with forward KL                       & 0.201    & 0.276    &  0.200     & 0.828           &     0.745                  &   1.606      &     32.594                 \\
                        & $\theta_{\text{bridge}}$ trained with reverse KL                      &  4.336    &   0.267    &    0.200    &    0.825       &       0.770                &    1.124     & 31.499   \\ \hline \rowcolor{blue!8} 
\multicolumn{9}{c}{databricks-dolly-15k~\cite{DatabricksBlog2023DollyV2}} \\  \hline \rowcolor{gray!20} 
\textbf{Full context}            & N/A    &  0.435   &   0.392   &  0.340   &     0.851     &  0.600            &   0.960   &  50.517 \\  \hline 
\textbf{WinStega}                & N/A                                  &  35.162    &  0.064     &   0.097    &    0.661      &        0.960            &   2.996    &  22.549               \\ \hline
  \textbf{ASW}  & 0 token        &   5.332   &   0.197    &  0.214    &   0.815        &     0.895         &   0.892     &    28.405                                     \\  \hline
          \multirow{2}{*}{\textbf{ASW (hard)}}              & \texttt{``[CONTEXT TRUNCATED]\textbackslash n''}                             &   5.165   &    0.214   &   0.226    &    0.823      &         0.810           &   0.986    &   28.872           \\
                        & \texttt{``...\textbackslash n''}                             &  5.244   &  0.208   &   0.219  &    0.817     &      0.865        &       0.933  &   27.874                  \\ \cline{1-9}
\multirow{3}{*}{\textbf{ASW (soft)}}   & Random untrained  $\theta_{\text{bridge}}$                           &  6.956   &  0.173   &   0.201  &  0.798       &          0.905             &      0.864   &   30.944               \\
                        & $\theta_{\text{bridge}}$ trained with forward KL                       &   1.085  &   0.217  &  0.218   &    0.820     &        0.825               &  1.534       &     29.921             \\
                        & $\theta_{\text{bridge}}$ trained with reverse KL                      & 5.084    &  0.214   &  0.225   &    0.825     &     0.805       &     1.047    & 29.941   \\ \hline \rowcolor{blue!8} 
\multicolumn{9}{c}{Super-NaturalInstructions~\cite{wang-etal-2022-super}} \\  \hline \rowcolor{gray!20} 
\textbf{Full context}            & N/A    &  0.669   &   0.333   &  0.342   &     0.839     &           0.720   &   0.819   & 13.790  \\  \hline 
\textbf{WinStega}                & N/A        &   22.171   &  0.024  &  0.077   &    0.644      &            0.985         &   3.003     &   22.519                  \\ \hline
 \textbf{ASW}  & 0 token                            &   11.507   &    0.105   &   0.162   &    0.786       &       0.925       &   0.980     &        32.340             \\ \hline
    \multirow{2}{*}{\textbf{ASW (hard)}}                     & \texttt{``[CONTEXT TRUNCATED]\textbackslash n''}                             &   12.622   &   0.107    &    0.164   &     0.792     &   0.880     &   1.030    &    37.687          \\
                        & \texttt{``...\textbackslash n''}         &   12.117  &  0.107   &  0.163   &    0.788     &            0.875           &   0.995      &     32.654                \\ \cline{1-9}
\multirow{3}{*}{\textbf{ASW (soft)}}   & Random untrained  $\theta_{\text{bridge}}$                           &  12.230   &  0.106   &  0.162   &     0.784    &          0.930             &    1.030     &   40.115               \\
                        & $\theta_{\text{bridge}}$ trained with forward KL                       &   5.786  &  0.109   &  0.165   &   0.789      &    0.860     &   1.750      &    38.240              \\
                        & $\theta_{\text{bridge}}$ trained with reverse KL          &  10.820   & 0.109    &  0.167   &    0.790     &          0.855             &     1.211    &  38.652  \\  \bottomrule[1pt]                  
\end{tabular}}
 \caption{Average results across various metrics under different datasets and different methods, where the setting is Qwen2.5-7B-Instruct, $l_{\text{bridge}} = 8$, and $w = 10$.}
 \label{table: Effect of dataset}
\end{table*}

\section{Analysis of Robustness in ASW Framework}
\label{appendix:Analysis of Robustness}

In this section, we analyze the robustness of the proposed ASW framework. Specifically, we focus on Bob’s extraction process based on autoregressive inference. To successfully extract a stegotext of length $T$, it is essential to evaluate the inference robustness under adversarial modifications.

\begin{proposition}\label{proposition: robustness against m-token modification}
If $m$ out of $T$ tokens are randomly modified, and a language model uses a left-context window of size $w$, the expected number $E$ of the positions where the autoregressive inference output is influenced is $T-1 - \frac{1}{\binom{T}{m}}\sum_{i=1}^{\min(w,T-1)}\binom{T-i}{m} -  \frac{1}{\binom{T}{m}} \max(0,T-1-w)\binom{T-w}{m}$.
\end{proposition}
\begin{proof}
There are $T$ tokens in the sequence, indexed from $0$ to $T-1$. A random subset of $m$ tokens is modified, chosen uniformly from the $\binom{T}{m}$. The model generates tokens autoregressively. For position $i$ (where $0 \leq i \leq T-1$), the output depends on a left-context window of size at most $w$, specifically, the tokens from $\max(0, i-w)$ to $i-1$. The output at position $i$ is influenced if at least one token in its left-context window is modified. Position $0$ has no left context, so it is never influenced. Thus, we consider positions $i =1$ to $T -1$.

Define the indicator random variable $X_i$ for $i = 1, \dots, T-1$:
\begin{equation}
X_i = 
\begin{cases}
  1, & \text{if the output at position $i$ is influenced} \\
  0, & \text{otherwise }
\end{cases}.
\end{equation}
The total number of influenced positions is $\sum_{i=1}^{T-1}X_i$. By linearity of expectation, the expected number of influenced positions is: $E = \sum_{i=1}^{T-1}\text{Pr}(X_i = 1)$.

The output at position $i$ is influenced if at least one token in its left-context window is modified. The context window for position $i$ is:
\begin{equation}
    C_i = \{j: \max(0,i-w)\leq j \leq i-1\}
\end{equation}
where the size of $C_i$ is $d_i = |C_i| = \min(i,w)$.

The event $X_i = 0$ occurs only if no token in $C_i$ is modified. The number of tokens outside $C_i$ is $T-d_i$. The probability that none of the $m$ modified tokens are in $C_i$ is:
\begin{equation}
    \text{Pr}(\text{no modified token in $C_i$}) = \frac{\binom{T-d_i}{m}}{\binom{T}{m}}
\end{equation}
provided $\binom{T-d_i}{m} = 0$ if $m > T- d_i$ (by convention for binomial coefficients). Thus:
\begin{equation}
    \text{Pr}(X_i = 1) = 1 - \frac{\binom{T-d_i}{m}}{\binom{T}{m}}.
\end{equation}

Substitute $Pr(X_i = 1)$ into the expectation:
\begin{align}
    & E = \sum_{i=1}^{T-1}\text{Pr}(X_i = 1) \notag \\
    = & \sum_{i=1}^{T-1}\left(1 - \frac{\binom{T-d_i}{m}}{\binom{T}{m}}\right)  \notag  \\
    = & T-1 - \frac{1}{\binom{T}{m}}\sum_{i=1}^{T-1}\binom{T-d_i}{m}.
\end{align}
\begin{itemize}
    \item For $i\leq w$ (i.e., $d_i = i$), the indices are $i = 1$ to $v$, where $v = \min(w,T-1)$.
    \item For $i > w$ (i.e., $d_i = w$), the indices are $i = v + 1$ to $T - 1$, but only if $w < T -1$ (otherwise, this range is empty).
\end{itemize}

Thus:
\begin{align}
    & \sum_{i=1}^{T-1}\binom{T-d_i}{m} \notag \\ = & \sum_{i=1}^{v}\binom{T-i}{m} + \sum_{i=v + 1}^{T-1}\binom{T-w}{m} \notag  \\ = &  \sum_{i=1}^{v}\binom{T-i}{m} + (T-1-v)\binom{T-w}{m}.
\end{align}

Since $v = \min(w, T - 1)$, we have:
\begin{align}
    T-1-v = &
    \begin{cases}
        T-1-w, & \text{if} \ w< T -1, \\
        0, & \text{if} \ w \geq T-1,
    \end{cases} \notag \\
     & = \max(0,T-1-w)
     .
\end{align}

Thus:
\begin{align}
    & \sum_{i=1}^{T-1}\binom{T-d_i}{m} \notag =  \sum_{i=1}^{v}\binom{T-i}{m} + \\ & \max(0,T-1-w)\binom{T-w}{m}.
\end{align}

Substitute back into $E = \sum_{i=1}^{T-1}\text{Pr}(X_i = 1)$:
\begin{align}
    &E =  \sum_{i=1}^{T-1}\text{Pr}(X_i = 1) \notag \\ = & T-1 - \frac{1}{\binom{T}{m}}\sum_{i=1}^{\min(w,T-1)}\binom{T-i}{m} \notag \\ - & \frac{1}{\binom{T}{m}} \max(0,T-1-w)\binom{T-w}{m}.
\end{align}

\end{proof}

\begin{proposition}\label{proposition: E(w)}
$E$ is non-decreasing when $w$ increases from $1$ to $T-1$.
\end{proposition}
\begin{proof}

Considering $1 \leq w \leq T-1$, $E$ is redefined as $E(w)$:
\begin{align}
    &E(w) =  \sum_{i=1}^{T-1}\text{Pr}(X_i = 1) \notag \\ = & T-1 - \frac{1}{\binom{T}{m}}\sum_{i=1}^{w}\binom{T-i}{m} \notag \\ - & \frac{1}{\binom{T}{m}} (T-1-w)\binom{T-w}{m}
\end{align}
where $T$ is the total number of tokens, $m$ is the number of modified tokens $(0 \leq m \leq T)$, and $w$ is the left-context window size. This expression holds for $w \in [1,T-1]$ because $v = \min(w,T-1) = w$ and $\max(0,T-1-w) = T -1 -w$ when $w \leq T -1$.

Considering $w \in \{1,2,\dots, T-2 \}$, define:
\begin{equation}
    S(w) = \sum_{i=1}^{w}\binom{T-i}{m} +  (T-1-w)\binom{T-w}{m}.
\end{equation}
Then:
\begin{equation}
    E(w) = T - 1 - \frac{S(w)}{\binom{T}{m}}.
\end{equation}
\begin{equation}
    S(w+1) - S(w) = -(T-w-1)\binom{T-w-1}{m-1}.
\end{equation}
\begin{align}
     E(w+1) - E(w) \notag  =  & - \frac{S(w+1) - S(w)}{\binom{T}{m}} \notag \\ = & \frac{(T-w-1)\binom{T-w-1}{m-1}}{\binom{T}{m}}.
\end{align}
Thus, $E(w+1) - E(w) \geq 0$ when $w \in {1,2,\dots,T-2}$ with \textbf{equality if and only if} $\binom{T-w-1}{m-1} = 0$. This occurs when:
\begin{itemize}
    \item $m-1<0$ (i.e., $m=0$), or
    \item $m-1>T-w-1$ (i.e., $m>T-w$).
\end{itemize}

\end{proof}

\section{Steganalysis Settings}
\label{sec: steganalysis}

The binary discriminator used for steganalysis is a fine-tuned version of the pretrained BERT~\cite{devlin-etal-2019-bert}. 
For the dataset, in each experimental group, we use 500 pairs of a reference text (via multinomial sampling) and a stegotext. In each group, the 1000 texts are split in a 6:2:2 ratio to create the training, validation, and test sets.
For fine-tuning, we use AdamW~\cite{loshchilov2018decoupled} as the optimizer with a learning rate of $1 \times 10^{-5}$. The batch size is set to 128, and the discriminator is trained for 5 epochs.
Experiments were implemented in Python 3.12.7 with Torch 2.5.0, and accelerated by using RTX 6000 Ada Generation GPUs.

\section{Effect of Model and Data Variants}
\label{app: Effect of Model and Data Variants}
Table~\ref{table: Effect of model} presents the average results across multiple metrics under different methods. The experiments were conducted using Qwen2.5-3B-Instruct, Qwen2.5-7B-Instruct, and Qwen2.5-14B-Instruct, with settings identical to those in the main experiments (Section~\ref{sec: Main Experiments}). From this table, we can draw several key observations:

\textit{1)} Across all three language models, ASW with a trained soft bridge context achieves the best performance, followed by ASW with a hard bridge context, then ASW without a bridge context, and finally WinStega. Moreover, ASW consistently and substantially outperforms WinStega across these language models.

\textit{2)} ASW becomes increasingly beneficial as the scale of the adopted language model grows. Specifically, compared with WinStega, we examine the extent to which ASW improves performance in BLEU, ROUGE-L, BERTScore, and steganalysis accuracy. We define the rate of change as $\frac{x_{\text{ASW}} - x_{\text{WinStega}}}{x_{\text{WinStega}}}$, where $x_{\text{ASW}}$ denotes the value of a given metric under ASW (any variant), and $x_{\text{WinStega}}$ denotes the corresponding value under WinStega.
Figure~\ref{fig: Scale_extent} illustrates the trajectories of these rates of change, showing that ASW exhibits greater potential with larger models.
In this figure, for ASW (hard) and ASW (soft), only the best value is plotted for each group.
One possible explanation is that larger language models have a deeper understanding of the bridge context, enabling them to better \textit{imagine the excluded content}.

Table~\ref{table: Effect of dataset} lists the average results when the evaluation dataset (500 samples) varies, from InstructionWild~\cite{instructionwild} to databricks-dolly-15k (open QA)~\cite{DatabricksBlog2023DollyV2} and Super-NaturalInstructions (purely questions)~\cite{wang-etal-2022-super}.
By comparison, under these three evaluation datasets, all variants of ASW consistently outperform the baseline, WinStega.
In addition, since the question (prompt) datasets differ, the number of tokens generated in response also varies, leading to differences in runtime, particularly when the inference is based on the full context.

To further validate the generalizability of ASW beyond the Qwen model family, we conduct additional experiments on \texttt{Llama-3.1-8B-Instruct}, a representative model from the LLaMA series with a substantially different architecture and training corpus. Following the same experimental protocol as in Section~\ref{sec: Intuition: Hard Bridge Context}, we evaluate the average KL divergence at each step $t$ under different context window sizes $w \in \{10, 20, 30, 40, 50\}$, using 500 sampled instances per setting. The results are summarized in Table~\ref{tab:cross-arch-llama}.

\begin{table*}[!t]
\renewcommand{\arraystretch}{1.0}
\centering
\scalebox{0.88}{
\begin{tabular}{l|l|ccccc}
\toprule[1.0pt]
 \textbf{Context window}  & \textbf{Hard bridge context}   & $w = 10$ & $w = 20$ & $w = 30$ & $w = 40$ & $w = 50$ \\ 
\midrule[1.0pt]
\textbf{Basic} (Figure~\ref{fig: WinStega_context}) & N/A  & \textcolor{red!75!black}{\textbf{2.134}}  & \textcolor{red!75!black}{\textbf{1.477}} &  \textcolor{red!75!black}{\textbf{1.291}}  &  \textcolor{red!75!black}{\textbf{1.051}}  &  \textcolor{red!75!black}{\textbf{0.875}}  \\ \hline
\multirow{7}{*}{\textbf{ASW} (Figure~\ref{fig: ASW_context})} & 0 token  & 1.635  & 1.034   & 0.855   & 0.634   & 0.572   \\ \cline{2-7}
 & Random 5 tokens  & 1.671 &  1.157  &  0.933  & 0.659   &  0.544  \\
 & Random 10 tokens  & 1.889 &  1.248  &  1.020  &  0.757  &  0.624   \\ \cline{2-7}
 & \texttt{``[Some texts are missing here]\textbackslash n''}  & \textcolor{green!50!black}{\textbf{1.387}}  & \textcolor{green!50!black}{\textbf{0.960}}   &  \textcolor{green!50!black}{\textbf{0.762}}  &  \textcolor{green!50!black}{\textbf{0.597}}  & 0.499   \\
 & \texttt{``[previous message removed]\textbackslash n''} & 1.427  & 0.969   &  0.798  & 0.613   & \textcolor{green!50!black}{\textbf{0.495}}   \\
 & \texttt{``[CONTEXT TRUNCATED]\textbackslash n''} & 1.421 &  0.969  &  0.795  &  0.617  & 0.500   \\
 & \texttt{``...\textbackslash n''} & 1.423 &  0.976  &  0.800  &  0.622  & 0.503  \\
\bottomrule[1.0pt]
\end{tabular}}
 \caption{Average KL divergence $D_{\text{KL}}(\boldsymbol{p}_{\text{full}}^{(t)}||\boldsymbol{p}_{\text{window}}^{(t)})$ at each step $t$ under different settings on \texttt{Llama-3.1-8B-Instruct} (lower is better). Each setting uses 500 sampled instances. \textcolor{red!75!black}{\textbf{Red}} values indicate the highest (the worst), and \textcolor{green!50!black}{\textbf{green}} values indicate the lowest (the best) in each column.}
 \label{tab:cross-arch-llama}
\end{table*}

As shown in Table~\ref{tab:cross-arch-llama}, the results on \texttt{Llama-3.1-8B-Instruct} are consistent with the key findings reported in Section~\ref{sec: Intuition: Hard Bridge Context}. First, \textbf{the choice of prompt matters}: ASW with well-designed hard bridge contexts consistently outperforms the basic context window across all window sizes, confirming that the adaptive sliding window mechanism generalizes effectively to non-Qwen architectures. Second, \textbf{the design of bridge context matters}: informative placeholders such as \texttt{``{[}Some texts are missing here{]}''} and \texttt{``{[}CONTEXT TRUNCATED{]}''} yield lower KL divergence than random token fillers or the zero-token baseline, indicating that semantically meaningful bridge contexts help the model better accommodate truncated histories regardless of the underlying architecture.

\begin{table}[!t]
\centering
\scalebox{0.85}{
\begin{tabular}{l|l|ccc}
\toprule[1pt]
                      \multicolumn{2}{c|}{Method}           & \multicolumn{1}{c}{$m = 1$} & \multicolumn{1}{c}{$m = 2$} & \multicolumn{1}{c}{$m = 3$} \\ \hline \rowcolor{gray!20} 
\multicolumn{2}{c|}{Full context} & 49.34\% & 34.22\% & 24.37\% \\ \hline
\multirow{2}{*}{$w = 10$} & WinStega &  96.29\%      &  92.82\%    &  89.85\%  \\
                        & ASW      & \textbf{97.91\%}      & \textbf{96.14\%}      & \textbf{94.03\%}         \\ \hline

\multirow{2}{*}{$w = 30$} & WinStega &            89.86\%            &        81.07\%                &      72.77\%               \\
& ASW      &    \textbf{94.23\%}   &   \textbf{88.89\%}    &   \textbf{83.65\%}      \\ \hline
\multirow{2}{*}{$w = 50$} & WinStega &            84.25\%               &     72.79\%                      &       58.04\%                    \\
                        & ASW      &     \textbf{90.01\%}        &      \textbf{82.43\%}   &   \textbf{74.70\%}    \\ \bottomrule[1pt]                   
\end{tabular}}
 \caption{The ratio of positions with unaffected inference when the number $m$ of deleted tokens and the latest token length $w$ vary. The deleted positions are also considered as affected.}
 \label{table: robustness_deletion}
\end{table}

\begin{table}[!t]
\centering
\scalebox{0.85}{
\begin{tabular}{l|l|ccc}
\toprule[1pt]
                      \multicolumn{2}{c|}{Method}           & \multicolumn{1}{c}{$m = 1$} & \multicolumn{1}{c}{$m = 2$} & \multicolumn{1}{c}{$m = 3$} \\ \hline \rowcolor{gray!20} 
\multicolumn{2}{c|}{Full context} & 49.68\% & 34.50\% & 24.51\% \\ \hline
\multirow{2}{*}{$w = 10$} & WinStega &  96.33\%      &  92.86\%    &  89.99\%  \\
                        & ASW      & \textbf{97.92\%}      & \textbf{96.24\%}      & \textbf{94.13\%}         \\ \hline

\multirow{2}{*}{$w = 30$} & WinStega &            89.82\%            &        80.95\%                &      72.92\%               \\
& ASW      &    \textbf{94.10\%}   &   \textbf{89.00\%}    &   \textbf{83.64\%}      \\ \hline
\multirow{2}{*}{$w = 50$} & WinStega &            84.42\%               &     72.77\%                      &       58.06\%                    \\
                        & ASW      &     \textbf{90.05\%}        &      \textbf{82.15\%}   &   \textbf{74.43\%}    \\ \bottomrule[1pt]                   
\end{tabular}}
 \caption{The ratio of positions with unaffected inference when the number $m$ of inserted tokens and the latest token length $w$ vary. The inserted new positions are also considered as affected.}
 \label{table: robustness_insertion}
\end{table}

\begin{algorithm}[!t]
\small
\caption{Steganographic embedding in ASW (with a hard bridge context)}\label{algorithm: Steganographic embedding in ASW (with hard bridge context)} 
\textbf{Input:}\\
Prompt, $\boldsymbol{s}_{\text{prompt}} = [s^{(-N_p)},\dots,s^{(-1)}]$ \\
Secret message, $m_s$  \\
Parameters of the language model, $\theta$ \\
Hard bridge context, $\boldsymbol{s}_{\text{bridge}}$ \\
Length of the latest tokens, $w$ \\
{\textbf{Output:}}\\
Steganographic text, $t_s$\\

\begin{algorithmic}[1]
\STATE $\boldsymbol{s}_{\text{generated}} \leftarrow []$;
\FOR{$t = 0,1,\dots$}
    \STATE $\boldsymbol{s}_{\text{latest}} \leftarrow \boldsymbol{s}_{\text{generated}}[-w:]$;
    \STATE $\boldsymbol{s}_{\text{window}} \leftarrow \boldsymbol{s}_{\text{prompt}} \Vert \boldsymbol{s}_{\text{bridge}} \Vert \boldsymbol{s}_{\text{latest}}$;
    \STATE $\boldsymbol{p}^{(t)}_{\text{window}} \leftarrow \text{softmax}(f_{\text{LM}}(\boldsymbol{s}_{\text{window}};\theta))$; 
    \STATE Use the steganographic embedding algorithm, $\boldsymbol{p}^{(t)}_{\text{window}}$,  and $m_s$ to generate the next token $s^{(t)}$;
    \STATE $\boldsymbol{s}_{\text{generated}} \leftarrow \boldsymbol{s}_{\text{generated}} \Vert s^{(t)}$;
\ENDFOR

\STATE Detokenize $\boldsymbol{s}_{\text{generated}}$ to $t_s$;
\RETURN $t_s$

\end{algorithmic}

\end{algorithm}

\begin{algorithm}[!t]
\small
\caption{Steganographic extraction in ASW (with a hard bridge context)}\label{algorithm: Steganographic extraction in ASW (with hard bridge context)} 
\textbf{Input:}\\
Prompt, $\boldsymbol{s}_{\text{prompt}} = [s^{(-N_p)},\dots,s^{(-1)}]$ \\
Steganographic text, $t_s$\\
Parameters of the language model, $\theta$ \\
Hard bridge context, $\boldsymbol{s}_{\text{bridge}}$ \\
Length of the latest tokens, $w$ \\
{\textbf{Output:}}\\
Secret message, $m_s^{\prime}$  \\

\begin{algorithmic}[1]
\STATE $m_s^{\prime} \leftarrow \emptyset$;
\STATE Tokenize $t_s$ to $\boldsymbol{s}_{\text{generation}} = [s^{(0)},s^{(1)},\dots]$;
\FOR{$t = 0,1,\dots$}
    \STATE $\boldsymbol{s}_{\text{latest}} \leftarrow \boldsymbol{s}_{\text{generated}}[t-w:t]$;
    \STATE $\boldsymbol{s}_{\text{window}} \leftarrow \boldsymbol{s}_{\text{prompt}} \Vert \boldsymbol{s}_{\text{bridge}} \Vert \boldsymbol{s}_{\text{latest}}$;
    \STATE $\boldsymbol{p}^{(t)}_{\text{window}} \leftarrow \text{softmax}(f_{\text{LM}}(\boldsymbol{s}_{\text{window}};\theta))$; 
    \STATE Use the steganographic extraction algorithm,  $\boldsymbol{p}^{(t)}_{\text{window}}$, and $s^{(t)}$ to update $m_s^{\prime}$;
\ENDFOR
\RETURN $m_s^{\prime}$

\end{algorithmic}

\end{algorithm}

\begin{figure*}[!t]
 \centering
 \includegraphics[width=\textwidth]{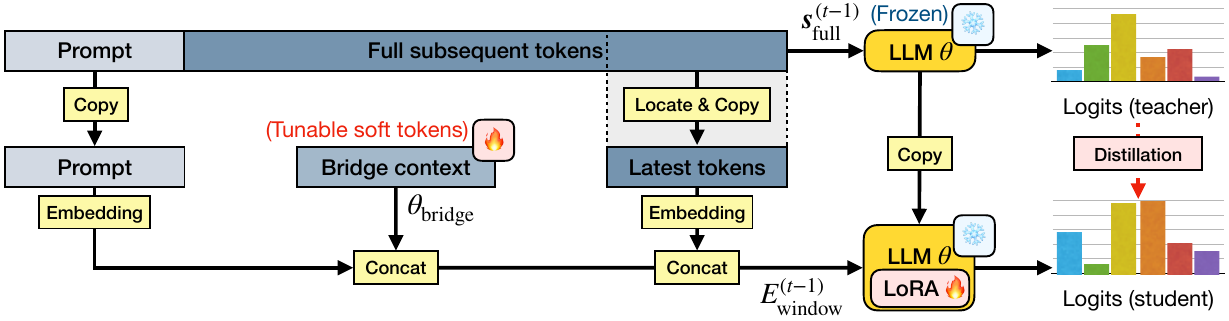} 
 \caption{An overview of the distillation framework based on LoRA. All the parameters of the language model are not tunable, and only the soft bridge context and the LoRA module are tunable. The distillation objective is to mitigate the gap between logits based on the full context $\boldsymbol{s}_{\text{full}}^{(t-1)}$ (teacher) and the logits based on a $E_{\text{window}}^{(t-1)}$ (student).}
 \label{fig: Self-disttilation_LoRA}
\end{figure*}

\begin{table*}[!t]
\renewcommand{\arraystretch}{1.0}
\centering
\scalebox{0.68}{
\begin{tabular}{l|l|cccc|c|cc}
\toprule[1pt]
\multirow{2}{*}{\textbf{Method}} & \multirow{2}{*}{\textbf{Bridge context}} & \multicolumn{4}{c|}{\textbf{Text quality}} & \textbf{Imperceptibility}      & \multicolumn{2}{c}{\textbf{Efficiency}} \\ \cline{3-9}
                        &                                    & $\Delta$PPL$\downarrow$  & BLEU$\uparrow$  & ROUGE-L$\uparrow$ & BERTScore$\uparrow$ & Steganalysis ACC$\downarrow$ & Capacity$\uparrow$     & Time (s)$\downarrow$     \\ \hline \rowcolor{gray!20} 
\textbf{Full context}            & N/A    &  1.090   &  0.405     &  0.257     &    0.856       &       0.560                &   1.024      &   115.43   \\  \hline
\textbf{WinStega}                & N/A                                &  31.486    &  0.082  &    0.098   &    0.699       &       0.955                &  1.995       &       22.250               \\ \hline \rowcolor{blue!8} 
\multicolumn{9}{c}{\textbf{w/o LoRA; w/o soft bridge context} (Trainable parameter size: 0)} \\  \hline
  \textbf{ASW}  & 0 token       &   5.157   &   0.220    &  0.190     &     0.806      &        0.885               &    0.966     &      29.702                \\ \hline
                 \multirow{2}{*}{\textbf{ASW (hard)}}       & \texttt{``[CONTEXT TRUNCATED]\textbackslash n''}                             &  3.246    &   0.242    &   0.195    &    0.816       &    0.830                   &   1.012      &   30.868                   \\
                        & \texttt{``...\textbackslash n''}                             &  4.560    &   0.248    &   0.195    &     0.816      &         0.815              &     1.035    &       30.505               \\ \hline \rowcolor{blue!8} 
\multicolumn{9}{c}{\textbf{w/o LoRA; w/ soft bridge context} (Trainable parameter size: 28,672)} \\  \hline
             \multirow{2}{*}{\textbf{ASW (soft)}}            & Trained with forward KL                       & 0.201    &   0.276    &  0.200     & 0.828           &   0.745                  &   1.606      &     32.594                 \\
                        & Trained with reverse KL                      &  4.336    &   0.267    &    0.200    &    0.825       &       0.770                &    1.124     & 31.499   \\ \hline \rowcolor{blue!8} 
\multicolumn{9}{c}{\textbf{w/ LoRA; w/o soft bridge context} (Trainable parameter size: 1,261,568)} \\ \hline
\multirow{2}{*}{\textbf{ASW (LoRA)}}            & Trained with forward KL                       &   0.312  &   0.281    &   0.193    &    0.827        &         0.750            &    1.631    &   29.238                  \\
                        & Trained with reverse KL      &   4.935   &  0.275    &   0.201     &  0.826        &     0.750     &    1.057    &  30.044   \\ \hline \rowcolor{blue!8} 
\multicolumn{9}{c}{\textbf{w/ LoRA; w/ soft bridge context} (Trainable parameter size: 1,290,240)} \\  \hline
\multirow{2}{*}{\begin{tabular}[c]{@{}l@{}}\textbf{ASW} \\ \textbf{(LoRA + soft)}\end{tabular}}            & Trained with forward KL           &   0.436    &   0.286    &   0.196    &       0.827     &    0.755       &   1.660     &      32.187               \\
                        & Trained with reverse KL      &   0.103   &   0.278   &     0.201   &    0.828      &     0.740      &   1.184     &  31.763  \\   \bottomrule[1pt]                  
\end{tabular}}
 \caption{Average results across various metrics under different methods (including the LoRA module), where the setting is Qwen2.5-7B-Instruct, $l_{\text{bridge}} = 8$, and $w = 10$.}
 \label{table: Result_lora}
\end{table*}

\section{Human Evaluation of Text Fluency}
\label{sec:human-eval}

While automated metrics such as BLEU and ROUGE-L provide useful quantitative assessments of text quality, they may not fully capture the perceived fluency and imperceptibility of steganographic text from a human perspective. To address this, we conduct a human evaluation study focusing on text fluency.

\paragraph{Setup.}
We randomly select 50 prompts and use \texttt{Qwen2.5-7B-Instruct} to generate four types of outputs for each prompt: (1)~covertext produced by random sampling, (2)~full-context stegotext, (3)~WinStega stegotext with $w = 30$, and (4)~ASW stegotext with the soft bridge context trained via forward KL and $w = 30$. Human evaluators are asked to assign a fluency score on a 5-point Likert scale (1~=~least fluent, 5~=~most fluent) without knowing the generation method behind each text. The average scores are summarized in Table~\ref{tab:human-eval}.

\begin{table}[t]
\renewcommand{\arraystretch}{1.0}
\centering
\scalebox{0.88}{
\begin{tabular}{l|c}
\toprule[1.0pt]
\textbf{Type} & \textbf{Fluency Score} \\
\midrule[1.0pt]
Covertext (random sampling) & 3.12 \\
Full-context stegotext & 3.04 \\
WinStega stegotext ($w = 30$) & 1.88 \\
Our ASW stegotext ($w = 30$) & \textbf{2.56} \\
\bottomrule[1.0pt]
\end{tabular}}
\caption{Human evaluation of text fluency (1--5 scale, higher is better). ASW stegotext substantially outperforms WinStega and approaches the fluency of full-context stegotext.}
\label{tab:human-eval}
\end{table}

\paragraph{Analysis.}
As shown in Table~\ref{tab:human-eval}, covertext and full-context stegotext receive comparable fluency scores (3.12 vs.\ 3.04), confirming that full-context steganographic encoding introduces minimal perceptible degradation. WinStega stegotext, by contrast, scores notably lower (1.88), reflecting the quality loss caused by the na\"ive context truncation strategy. Our ASW stegotext achieves a fluency score of 2.56, substantially outperforming WinStega and narrowing the gap with full-context stegotext. 

\section{Supplementary Robustness Scenarios}
\label{sec: Supplementary Robustness Scenarios}
This section supplements Section~\ref{sec: Robustness Scenarios}.  
At first, we justify our use of the ratio of positions with unaffected inference, a proxy measure not directly tied to the embedded bits, as an indicator of robustness.
\begin{itemize}
    \item \textbf{Generalizability.} Linguistic steganographic methods differ in how they encode bits and sampling, but the common point is that they all depend on the stability of the LM conditional distribution at each generative step. The ratio of positions with unaffected inference captures exactly this stability regardless of the specific encoding strategy. 
    Therefore, this ratio functions as a fundamental \textbf{method-agnostic} indicator of robustness.
    \item \textbf{Availability.} Classical metrics such as bit error rates (BER) cannot be used directly when the extracted message length differs from the original. In addition, when the encoded candidate pool is not the entire vocabulary, it is possible to encounter failures in extracting an attacked stegotext. 
    In view of these issues, we adopt the ratio of positions with unaffected inference as a stably available metric to indicate robustness in widely diverse cases.
\end{itemize}

For deletion scenarios, random $m$ tokens were deleted from the original generated stegotext, and Table~\ref{table: robustness_deletion} lists robustness results when $(w,m)$ pairs vary.
For insertion scenarios, random $m$ tokens were randomly inserted into the original generated stegotext, and Table~\ref{table: robustness_insertion} lists robustness results when $(w,m)$ pairs vary.
Similar to Table~\ref{table: robustness} in Section~\ref{sec: Robustness Scenarios}, the superiority of robustness of ASW is consistently shown compared to other methods.

\section{Case Study: Self-Distillation based on LoRA}
As discussed in Section~\ref{sec: Tunable Soft Bridge Context via Self-Distillation} and the subsequent sections, a tunable soft bridge context via self-distillation is an effective way to reduce the divergence between inference within the full context and inference within a limited context window. This approach is inspired by the intuition behind the hard bridge context (Section~\ref{sec: Intuition: Hard Bridge Context}). However, a tunable soft-bridge context is not strictly required to implement self-distillation; it can alternatively be achieved by tuning model parameters or by employing an adapter. In this section, we examine the effects of employing LoRA (Low-Rank Adaptation)~\cite{hu2022lora} within our ASW framework.

The LoRA-based distillation framework is illustrated in Figure~\ref{fig: Self-disttilation_LoRA}. The training procedure follows the approach described in Section~\ref{sec: Tunable Soft Bridge Context via Self-Distillation}, with the key difference being the injection of trainable low-rank decomposition matrices into each layer of the Transformer architecture (i.e., the LoRA module). In the following experiments, we set the LoRA rank to 4, dropout to 0.1, and the scaling factor to 16. All other experimental setups and training–evaluation strategies follow Section~\ref{sec: General Setup}. 

Table~\ref{table: Result_lora} reports the average results under different adaptation methods. The results show slight improvements when incorporating LoRA compared to using only a soft bridge context. However, LoRA introduces a substantially larger number of trainable parameters (over 1 million) compared to just 28,672 parameters for eight trainable soft tokens ($l_{\text{bridge}} \times e = 8 \times 3584 = 28,672$). These findings suggest that lightweight training strategies can nearly achieve the upper bound of performance in this task, while \textit{increasing the number of trainable parameters yields only marginal improvements}.

\section{Robustness in Practical Implementations}
In addition to external alterations, in this section, we introduce two more practical issues which can affect robustness of linguistic steganography.

\subsection{Tokenization Inconsistency}
\label{appendix: Internal Tokenization Inconsistency}
In scenarios of linguistic (LM-based) steganography, there is a detokenization-retokenization pipeline between Alice and Bob. Specifically, Alice detokenizes the generated steganographic tokens into a stegotext and send it to Bob; Bob retokenizes the received stegotext into steganographic tokens for further extraction. However, Alice's steganographic tokens may not be the same as Bob's, which can be caused by irregularities in the training process, such as underrepresentation in training~\cite{geiping2024coercing}. In addition, a term ``unreachable tokens'' explicitly introduces that some tokens are never produced as a result of tokenizing text~\cite{land2024fishing}.

To address this issue for steganography, the existing countermeasures~\cite{nozaki-murawaki-2022-addressing,yan2023A,10831370,qi2024provably} let Bob get rid of the tokenization usage. 
Bob can directly extract the row stegotext  (instead of tokens) correctly ($m_s = m_s^{\prime}$). Recently,~\citet{yan2025addressingtokenizationinconsistencysteganography} has first addressed the tokenization inconsistency in steganography.
However, all of these methods do not consider robustness issues caused by adversarial alterations.

\subsection{Hardware Indeterminism}

Hardware indeterminism~\cite{atil2025nondeterminismdeterministicllmsettings} refers to the phenomenon where LLM outputs and the probability distribution of the next token, given the same context, can vary across different hardware environments and configurations. Such variability poses serious risks for reliable message extraction in real-world deployments, as even minor shifts in probability values may lead to incorrect decoding. This challenge highlights the need to explicitly address hardware-level robustness as a distinct design dimension for practical steganographic systems in future research.

\section{Why Forward KL Achieves Both High Quality and High Capacity}
\label{sec:forward-reverse-kl}

A natural question arises from our soft bridge results: \emph{why does Forward KL training yield both high text quality and high embedding capacity simultaneously, whereas Reverse KL consistently leads to lower capacity?} We provide a theoretical explanation grounded in the well-known asymmetry between the two divergence directions~\cite{minka2005divergence,bishop2006pattern}.

\paragraph{Forward KL is mass-covering.}
Forward KL penalizes the bridge distribution whenever it assigns insufficient probability to any region where the full-context distribution places non-negligible mass. As a result, the trained soft bridge tends to preserve and even amplify the full support and entropy of the teacher distribution. Because embedding capacity in linguistic steganography is directly related to the number of plausible candidate tokens with reasonable probability, this mass-covering behavior naturally maintains a larger set of admissible tokens at each generation step, leading to higher capacity.

\paragraph{Reverse KL is mode-seeking.}
Reverse KL allows the bridge distribution to concentrate heavily on a few dominant modes of the teacher distribution while collapsing low-probability regions. This narrows the support of the conditional distribution and substantially reduces entropy, leaving far fewer high-probability candidate tokens available for embedding. Consequently, Reverse KL consistently leads to lower embedding capacity.

\paragraph{Summary.}
The higher capacity of the Forward-KL-trained soft bridge is a direct consequence of its mass-covering optimization objective, which supports distributional diversity across the vocabulary, whereas Reverse KL's mode-seeking nature inherently restricts the usable token set. Regarding text-quality metrics, although Forward KL achieves better perplexity (PPL), its performance is close to that of Reverse KL on other metrics (particularly ROUGE-L). Therefore, the ``convenient'' advantage of Forward KL is primarily reflected in its higher embedding capacity rather than in notably superior text quality.

\begin{table*}[!t]
\renewcommand{\arraystretch}{1.0}
\centering
\scalebox{0.72}{
\begin{tabular}{l|l|cccc|c|cc}
\toprule[1pt]
 \multirow{2}{*}{\textbf{Soft bridge context}} & \multirow{2}{*}{$l_{\text{bridge}}$}  & \multicolumn{4}{c|}{\textbf{Text quality}} & \textbf{Imperceptibility}      & \multicolumn{2}{c}{\textbf{Efficiency}} \\ \cline{3-9}
                        &                                    & $\Delta$PPL$\downarrow$  & BLEU$\uparrow$  & ROUGE-L$\uparrow$ & BERTScore$\uparrow$ & Steganalysis ACC$\downarrow$ & Capacity$\uparrow$     & Time (s)$\downarrow$     \\ \hline \rowcolor{gray!20} 

0 token &  0      &   5.157   &   0.220    &  0.190     &     0.806      &        0.885               &    0.966     &      29.702              \\ \hline

\multirow{5}{*}{$\theta_{\text{bridge}}$ trained with forward KL}   &    1     &                    3.769   &  0.231    &  0.192    &   0.809       &   0.860   &   0.984   &  30.458     \\
&    4     &    1.394   &  0.268   &  0.197    &  0.824    &  0.795    &  1.534    &  30.744  \\
&   8      &    0.201     &  0.276   &  0.200    &  0.828  &     0.745     &  1.606    &  32.594          \\
&   16      & 0.939 &  0.276   &   0.199   &  0.827    &   0.755       &  1.712    &    37.501        \\
&   32      & 1.614 &  0.278   &    0.198  &   0.828   &  0.755   &  1.706    &   50.610   \\ \hline
\multirow{5}{*}{$\theta_{\text{bridge}}$ trained with reverse KL}   &    1     &    4.238               &  0.229   &   0.193   &  0.808    &  0.845   &  0.979    &  30.099          \\
&     4    &    4.288 &  0.264   &  0.197    &  0.822    &  0.810   &  1.029    & 30.889 \\
&     8    &   4.336 &  0.267   &  0.200    &  0.825    &    0.770      &   1.124   &   31.499         \\
&     16    &   4.477 &  0.267   &   0.201   & 0.826     &  0.775   &  1.129    &   37.095    \\
&     32   &    4.918 &  0.266   &  0.202    &  0.826    &    0.750      &  1.120    & 50.305      \\ \bottomrule[1pt]                  
\end{tabular}}
 \caption{Average results across various metrics under different lengths of the soft bridge context in our ASW framework, where the setting is Qwen2.5-7B-Instruct, and $w = 10$.}
 \label{table: Effect of l}
\end{table*}

\begin{table*}[!t]
\renewcommand{\arraystretch}{1.0}
\centering
\scalebox{0.68}{
\begin{tabular}{l|l|cccc|c|cc}
\toprule[1pt]
\multirow{2}{*}{\textbf{Method}} & \multirow{2}{*}{\textbf{Bridge context}} & \multicolumn{4}{c|}{\textbf{Text quality}} & \textbf{Imperceptibility}      & \multicolumn{2}{c}{\textbf{Efficiency}} \\ \cline{3-9}
                        &                                    & $\Delta$PPL$\downarrow$  & BLEU$\uparrow$  & ROUGE-L$\uparrow$ & BERTScore$\uparrow$ & Steganalysis ACC$\downarrow$ & Capacity$\uparrow$     & Time (s)$\downarrow$     \\ \hline \rowcolor{gray!20} 
\textbf{Full context}            & N/A    &  1.090   &  0.405     &  0.257     &    0.856       &       0.560                &   1.024      &   115.43   \\  \hline \rowcolor{blue!8} 
\multicolumn{9}{c}{$w = 10$} \\ \hline
\textbf{WinStega}                & N/A                                &  31.486    &  0.082  &    0.098   &    0.699       &       0.955               &  1.995       &       22.250               \\ \hline
 \textbf{ASW}   & 0 token                            &   5.157   &   0.220    &  0.190     &     0.806      &        0.885               &    0.966     &      29.702                \\ \hline
            \multirow{3}{*}{\textbf{ASW (hard)}}            & \texttt{``[CONTEXT TRUNCATED]\textbackslash n''}                             &  3.246    &   0.242    &   0.195    &    0.816       &    0.830                   &   1.012      &   30.868                   \\
                        & \texttt{``...\textbackslash n''}                             &  4.560    &   0.248    &   0.195    &     0.816      &         0.815              &     1.035    &       30.505               \\ \cline{1-9}
\multirow{3}{*}{\textbf{ASW (soft)}}   & Random untrained  $\theta_{\text{bridge}}$                           &  4.066    &   0.200    &   0.185    &    0.797       &    0.910      &     0.980    &    31.819                  \\
                        & $\theta_{\text{bridge}}$ trained with forward KL                       & 0.201    & 0.276    &  0.200     & 0.828           &     0.745                  &   1.606      &     32.594                 \\
                        & $\theta_{\text{bridge}}$ trained with reverse KL                      &  4.336    &   0.267    &    0.200    &    0.825       &       0.770                &    1.124     & 31.499   \\ \hline \rowcolor{blue!8} 
\multicolumn{9}{c}{$w = 30$} \\ \hline
\textbf{WinStega}                & N/A         & 10.290    & 0.190  &  0.150  &  0.771   &   0.930  & 1.676  & 24.262                \\ \hline
 \textbf{ASW}  & 0 token        &  4.578   & 0.335  &  0.219  & 0.841    &  0.725   & 1.087  & 46.682       \\ \hline
          \multirow{2}{*}{\textbf{ASW (hard)}}               & \texttt{``[CONTEXT TRUNCATED]\textbackslash n''}       &  1.068   & 0.351  &  0.224  &  0.844   &  0.720   & 1.116  &  48.654 \\
                        & \texttt{``...\textbackslash n''}      &  2.331   & 0.345  & 0.222   &  0.842   &  0.715   & 1.142  &  46.230             \\ \cline{1-9}
\multirow{3}{*}{\textbf{ASW (soft)}}   & Random untrained  $\theta_{\text{bridge}}$                          &  2.025   & 0.324  &  0.215  &  0.837   &   0.725  & 1.066  & 48.925                \\
                        & $\theta_{\text{bridge}}$ trained with forward KL                       & 0.301   & 0.362  & 0.229   &  0.847   &  0.695   & 1.320  &  48.529       \\
                        & $\theta_{\text{bridge}}$ trained with reverse KL                      &   3.979  & 0.357  &  0.230  &   0.847  &  0.705   & 1.099 & 47.708 \\ \hline \rowcolor{blue!8} 
\multicolumn{9}{c}{$w = 50$} \\ \hline
\textbf{WinStega}                & N/A     & 7.289   &  0.255  &  0.178  &  0.801   &  0.905   & 1.335 &  29.635          \\ \hline
  \textbf{ASW}  & 0 token  &  0.862   & 0.363  &  0.230  & 0.846    &  0.685   & 1.108 & 48.457        \\ \hline
         \multirow{2}{*}{\textbf{ASW (hard)}}               & \texttt{``[CONTEXT TRUNCATED]\textbackslash n''}                             &  3.794   & 0.364  &  0.232  &  0.848   &  0.670   & 1.128 & 49.519   \\
                        & \texttt{``...\textbackslash n''}     &  0.544   & 0.361  &  0.229  &  0.846   &   0.675  & 1.138  &  50.504            \\ \cline{1-9}
\multirow{3}{*}{\textbf{ASW (soft)}}   & Random untrained  $\theta_{\text{bridge}}$                           &  1.397   & 0.355  &  0.227  &   0.845  &  0.700   & 1.086 &  49.183             \\
                        & $\theta_{\text{bridge}}$ trained with forward KL                       &  0.195   & 0.370  &  0.235  &  0.850   &  0.670   & 1.178  & 50.319                 \\
                        & $\theta_{\text{bridge}}$ trained with reverse KL                      &   2.741  & 0.368  &  0.235  &  0.849   &   0.665  & 1.115  & 49.986  \\ \bottomrule[1pt]                  
\end{tabular}}
 \caption{Average results across various metrics under different lengths of the latest tokens and different methods, where the setting is Qwen2.5-7B-Instruct, and $l_{\text{bridge}} = 8$.}
 \label{table: Effect of w}
\end{table*}

\end{document}